\documentclass{article}

\usepackage[final,nonatbib]{neurips_2022}

\usepackage[utf8]{inputenc} 
\usepackage[T1]{fontenc}    
\usepackage[pagebackref=true,breaklinks=true,colorlinks,citecolor=gray,urlcolor=black]{hyperref}       
\usepackage{url}            
\usepackage{booktabs}       
\usepackage{amsfonts}       
\usepackage{nicefrac}       
\usepackage{microtype}      
\usepackage{xcolor}         
\usepackage{graphicx}
\usepackage{wrapfig}
\usepackage{amsthm}
\usepackage{dsfont}
\usepackage{mathtools}

\usepackage{amsmath,amsfonts,bm}

\def\eqref#1{equation~\ref{#1}}

\def\1{\bm{1}}

\def\vc{{\bm{c}}}

\def\vg{{\bm{g}}}
\def\vh{{\bm{h}}}

\def\vm{{\bm{m}}}

\def\vt{{\bm{t}}}

\def\vv{{\bm{v}}}

\def\vx{{\bm{x}}}

\def\mC{{\bm{C}}}

\def\mG{{\bm{G}}}

\def\mI{{\bm{I}}}

\def\mM{{\bm{M}}}

\def\mO{{\bm{O}}}

\def\mS{{\bm{S}}}

\def\mU{{\bm{U}}}
\def\mV{{\bm{V}}}

\def\mX{{\bm{X}}}

\def\mZ{{\bm{Z}}}

\DeclareMathAlphabet{\mathsfit}{\encodingdefault}{\sfdefault}{m}{sl}
\SetMathAlphabet{\mathsfit}{bold}{\encodingdefault}{\sfdefault}{bx}{n}

\def\gE{{\mathcal{E}}}
\def\gF{{\mathcal{F}}}

\def\gN{{\mathcal{N}}}

\def\sR{{\mathbb{R}}}

\def\sZ{{\mathbb{Z}}}

\newcommand{\R}{\mathbb{R}}

\usepackage{amssymb}

\newtheorem{theorem}{Theorem}
\newtheorem{proposition}{Proposition}
\newtheorem{lemma}[theorem]{Lemma}
\newtheorem{definition}{Definition}

\definecolor{DarkBlue}{rgb}{0,0.08,1}

\title{Learning Physical Dynamics with \\ Subequivariant Graph Neural Networks}

\author{%
  Jiaqi Han\\
  Tsinghua University\\
  \And 
  Wenbing Huang\thanks{Corresponding author: Wenbing Huang.}\\
  Gaoling School of Artificial Intelligence, \\
  Renmin University of China \\ Beijing Key Laboratory of Big Data\\ Management and Analysis Methods \\
  \And  
Hengbo Ma \\
 University of California, Berkeley\\
  \And 
  Jiachen Li\\
   Stanford University\\
  \And 
    Joshua B. Tenenbaum\\
   MIT BCS, CBMM, CSAIL\\
  \And 
  Chuang Gan\\
  UMass Amherst \\ MIT-IBM Watson AI Lab 
}

\begin{document}

\maketitle

\begin{abstract}

Graph Neural Networks (GNNs) have become a prevailing tool for learning physical dynamics. However, they still encounter several challenges: 1) Physical laws abide by symmetry,  which is a vital inductive bias accounting for model generalization and should be incorporated into the model design. Existing simulators either consider insufficient symmetry, or enforce excessive equivariance in practice when symmetry is partially broken by gravity. 2) Objects in the physical world possess diverse shapes, sizes, and properties, which should be appropriately processed by the model. To tackle these difficulties, we propose a novel backbone, \emph{Subequivariant Graph Neural Network}, which 1) relaxes equivariance to subequivariance by considering external fields like gravity, where the universal approximation ability holds theoretically; 2) introduces a new subequivariant object-aware message passing for learning physical interactions between multiple objects of various shapes in the particle-based representation; 3) operates in a hierarchical fashion, allowing for modeling long-range and complex interactions. Our model achieves on average over 3\% enhancement in contact prediction accuracy across 8 scenarios on Physion and 2$\times$ lower rollout MSE on RigidFall compared with state-of-the-art GNN simulators, while exhibiting strong generalization and data efficiency. Code and videos are available at our project page: \href{https://hanjq17.github.io/SGNN/}{https://hanjq17.github.io/SGNN/}.

\end{abstract}

\section{Introduction}


Learning and predicting the complicated physical dynamics and interactions between objects are vital to many tasks including physical reasoning~\cite{chen2022comphy,yi2019clevrer,chen2021grounding,lerer2016learning}, scene understanding~\cite{battaglia2013simulation}, and model-based planning and control ~\cite{li2022contact,lin2022diffskill,ma2022risp,ye2020object,agrawal2016learning}. Several learning-based differentiable simulators~\cite{sanchez2018graph,sanchez2020learning,li2018learning,pfaff2021learning} have been proposed, and they obtain promising achievements in simulating various kinds of interacting objects including rigid and fluids. The success mainly relies on the prevalence of graph neural networks~\cite{gilmer2017neural,sanchez2020learning,hamilton2017inductive}, which are desirable tools for physical simulation, by modeling particles as nodes, physical relations as edges, and their interactions as the message passing thereon.  

Notably, physical world exhibits certain symmetries.
For example, the way dominoes fall from left to right will exactly be preserved when the system is rotated horizontally to another direction. Inherently, humans reason about the dynamics with the preservation of such symmetry, and this inductive bias has been endorsed by the physical laws that also abide by the symmetries in our 3D world. 
Yet, regardless of this inductive bias, current differentiable simulators like GNS~\cite{sanchez2020learning} and DPI~\cite{li2018learning} would fail to learn the real dynamics that can generalize to all directions. In the example of dominoes, if the data in training are positioned from left to right, GNS can predict well during testing when the dominoes also fall in the same direction, but performs poorly if the scene is rotated horizontally (see demo in the \href{https://hanjq17.github.io/SGNN/}{Supplementary Video}). This observation implies that GNS overfits training samples without learning the true dynamics that obey the symmetry, limiting its generalization.

With this consideration, there have been a number of works, named geometrically equivariant graph neural networks~\cite{han2022geometrically}, that leverage symmetry as an inductive bias in learning to simulate. These models are designed such that their outputs will rotate/translate/reflect in the same way as the inputs, hence retaining the symmetry. However, models like EGNN~\cite{satorras2021en} and GMN~\cite{huang2022equivariant} are exerted E(3)-equivariance, the full symmetry of the 3D Euclidean space. Such constraint is so strong that it penalizes all directions in the 3D space, which cannot be applied to scenarios with external fields, like gravity. The existence of gravity breaks the symmetry in the vertical direction, reducing E(3) to its subgroup. We formally characterize this phenomenon of equivariance relaxation as \emph{subequivariance}. 

In general, simulating physics on datasets like Physion~\cite{bear2021physion} are highly challenging, due to the scale of the system (on average thousands of particles per system), the diversity of the interactions (\emph{e.g.}, collision, friction, gravity), as well as multiple shapes, materials, or even rigidness of the objects. These factors require unique efforts in designing the simulators, which brings another weakness of current GNN-based methods: From a practical point of view, they seldom explicitly involve the geometric object information into message passing. Moreover, when modeling  multiple interacting objects of different shapes, the interactions between or within objects are usually different. For example, the former aims to exchange momentum or energy across objects, while the latter
usually accounts for the geometrical constraint. It is thus important to involve objectness to distinguish interactions between particles within objects from those across objects.

In this work, we propose Subequivariant Graph Neural Networks (SGNN) that consist of several features to tackle the above challenges. \textbf{1.} We relax equivariance to subequivariance and design subequivariant functions to physical scenarios with the existence of gravity, and excitingly, we have proved that our designed form inherits the approximation universality. \textbf{2.} We formulate a novel subequivariant object-aware message passing framework for modeling dynamics and interactions between objects of varying shapes. \textbf{3.} We incorporate the subequivariant message passing into a hierarchical model to deal with long-range and complicated object interactions. We demonstrate the efficacy of our SGNN for learning physical dynamics on a large-scale challenging dataset Physion with 8 different scenarios, and a 3-object RigidFall dataset. Experimental results show that our model is capable of yielding more accurate dynamics prediction, is highly data-efficient, and has strong generalization compared with the state-of-the-art learning-based differentiable physical simulators.



\section{Related Work}

\textbf{GNN-based physical dynamics simulators.} There have been many works that employ graph neural networks as physical simulators of dynamical systems~\cite{battaglia2016interaction,mrowca2018flexible,kipf2018neural}. 
Graph Network Simulator (GNS) proposed by~\cite{sanchez2020learning} has been showcased a simple yet powerful tool in simulating large systems in particle-based representation of rigid and fluids by dynamically constructing interaction graph and performing multiple steps of information propagation. DPI~\cite{li2018learning} adds one-level of hierarchy to the rigid and predicts the rigid transformation via generalized coordinates, with applications to manipulation. There are also works that incorporate strong physical priors like Hamiltonian mechanics~\cite{sanchez2019hamiltonian,zhong2021extending}, geometrical constraints~\cite{rubanova2022constraintbased} and shapes~\cite{pfaff2021learning}. Despite the promising empirical results, these works have not given sufficient considerations on symmetry especially when external force like gravity presents, leaving room for enhancing their generalization to unseen testing data.

\textbf{Geometrically equivariant graph neural networks.} Equivariant graph neural networks~\cite{han2022geometrically} are a family of GNNs that are specifically designed to meet the constraint of certain symmetry, mostly involving translations, rotations, and/or reflections in Euclidean space. This goal is approached by several measures, including solving group convolution with irreducible representation~\cite{thomas2018tensor,fuchs2020se3} or leveraging invariant scalarization~\cite{villar2021scalars} like taking the inner product~\cite{satorras2021en,huang2022equivariant}. Our approach also belongs to the scalarization family together with EGNN~\cite{satorras2021en} and GMN~\cite{huang2022equivariant}. Specifically, EGNN models the interaction with the invariant distance as input, computed as an inner product of the relative positions, and GMN generalizes to a multi-channel version by taking into consideration a stack of multiple vectors. These equivariant GNNs operate on particle graphs or point clouds, while we additionally consider a combination of both particle- and object-level information for physical simulation. More importantly, they are assumed a full Euclidean symmetry with strong equivariance constraint, while we elaborate how to relax the constraint in the existence of external fields like gravity by leveraging subequivariance. 

\textbf{Equivariance on subgroups.} E($n$)-Steerable CNNs~\cite{weiler2019general,cesa2021program} develop convolutional kernels that meet equivariance on E(3)/E(2) and their subgroups by leveraging restricted representation. EMLP~\cite{finzi2021practical} obtains equivariance on arbitrary matrix groups. However, these approaches require specifying the particular group/subgroup to perform equivariance, while our goal here is to relax the equivariance on subgroups of E(3) by considering external force field, having more physical implications. Moreover, these works rely on computationally expensive operations like irreducible representation or solving group constraints, while our formulation resorts to scalarization, which is easy to implement, efficient to compute, and also comes with necessary universality guarantee.

\section{Subequivariant Graph Neural Networks}
\label{sec:SEGIN}

\subsection{Background}
\textbf{GNN-based simulators.}
We start by introducing the mechanism of GNN-based simulators for learning physical dynamics. We consider the particle-based representation of a physical system with $N$ particles consisting of $M$ objects. At time $t$, each particle within the system possesses some state information, including \textbf{1.} the position $\vec{\vx}_i^{(t)}\in\sR^3$ and the velocity $\vec{\vv}_i^{(t)}\in\sR^3$, which are both directional vectors; \textbf{2.} some attributes $\vh_i\in\sR^n$ without geometric context, such as the rigidness; \textbf{3.} the dynamic spatial connections with other particles, where an edge will be constructed if the distance between two particles is smaller than a threshold $r$, namely, $\gE^{(t)}=\{(i,j): \|\vec{\vx}_i^{(t)}-\vec{\vx}^{(t)}_j\|_2 < r\}$. We denote the object category of particle $i$ as $o(i)\in\sZ$ satisfying $o(i)=k$ if particle $i$ belongs to object $k$.
The goal here is to predict the position of the next step $\vec{\vx}_i^{(t+1)}$ given the above system information at time $t$
, which can be favorably modeled by GNNs, \emph{i.e.},  $\vec{\vx}_i^{(t+1)}=\varphi_{\text{GNN}}\left(\{\vec{\vx}_i^{(t)}\}, \{\vec{\vv}_i^{(t)}\}, \{{\vh}_i\}, \gE^{(t)}\right)$. Since the prediction at different time shares the same model, we will henceforth omit the temporal superscript $t$ for all variables for brevity.

The conventional GNN simulators $\varphi_{\text{GNN}}$~\cite{li2018learning,sanchez2020learning,mrowca2018flexible} offer a favorable solution by leveraging message-passing on the interaction graph, which computes:
\begin{align}
    \vm_{ij} &= \phi\left(\vec{\vx}_i, \vec{\vv}_i, \vec{\vx}_j, \vec{\vv}_j, \vh_i, \vh_j \right), \\
    \vec{\vx}'_i, \vec{\vv}'_i, \vh'_i &= \psi\left(\sum\nolimits_{j\in\gN(i)}\vm_{ij}, \vec{\vx}_i, \vec{\vv}_i, \vh_i \right),
\end{align}
where $\gN(i)=\{j: (i,j)\in\gE\}$ is the neighbors of node $i$, and $\phi$, $\psi$ are the edge message function and node update function, respectively. The prediction is obtained by conducting several iterations of message passing. Nevertheless, such form does not guarantee the desirable symmetry context with common choices of $\phi$ and $\psi$, and meanwhile the object information has not been elaborated, leaving room for a more exquisite  message passing scheme for learning complex physical dynamics. Recent works such as EGNN~\cite{satorras2021en} and GMN~\cite{huang2022equivariant} have actually considered E(3)-equivariance in the design of the functions $\phi$ and $\psi$ to pursue symmetry. However, they are not applicable to the case when symmetry is partially violated by gravity, which motivates our proposal in~\textsection~\ref{sec:SOMP}. Prior to introducing our work, we first provide necessary preliminaries related to equivariance.  

\textbf{Equivariance.}
In this paper, we basically focus on equivariance in terms of $\text{E(3)}$ transformations: translation, rotation, and reflection.  The formal definition is provided below.
\begin{definition}[Equivariance and Invariance]
\label{def:equ}
We call that the function $f:\R^{3\times m}\times\R^{n}\rightarrow \R^{3 \times m'}$ is $\text{E(3)}$-equivariant, if for any transformation $g$ in $\text{E(3)}$, $f(g\cdot\Vec{\mZ},\vh)=g\cdot f(\Vec{\mZ},\vh)$, $\forall\Vec{\mZ}\in\R^{3\times m}$, $\forall\vh\in\R^{n}$. Similarly, $f$ is invariant if $f(g\cdot\Vec{\mZ},\vh)= f(\Vec{\mZ},\vh)$, $\forall g\in \text{E(3)}$.  
\end{definition}

In Definition~\ref{def:equ}, the group action $\cdot$ is instantiated as $g\cdot\Vec{\mZ}\coloneqq \mO\Vec{\mZ}$ for the orthogonal transformation (rotation and reflection) where $\mO\in O(3)\coloneqq\{\mO\in\R^{3\times3}|\mO^\top\mO=\mI\}$, and $g\cdot\Vec{\mZ}\coloneqq \Vec{\mZ}+\vt$ for translation where $\vt\in\R^{3}$. Note that for the input of $f$, we have added the the right-arrow superscript on $\Vec{\mZ}$ to distinguish it from the scalar $\vh$ that is unaffected by the transformation, akin to our notations for the position $\vec{\vx}_i$, velocity $\vec{\vv}_i$, and attribute $\vh_i$.

It is non-trivial to derive equivariant function particularly for orthogonal transformations. GMN~\cite{huang2022equivariant} proposes a multichannel scalarization form, which yields,
\begin{align}
\label{eq:GMN}
    f(\Vec{\mZ},\vh)\coloneqq\Vec{\mZ}\mV, \quad\text{s.t.} \mV=\sigma(\Vec{\mZ}^{\top}\Vec{\mZ},\vh), 
\end{align}
where the inner product $\vec{\mZ}^\top\vec{\mZ}\in\sR^{m\times m}$ is firstly computed and concatenated with $\vh$, the resultant invariant term is then transformed by a function (usually a Multi-Layer Perceptron (MLP)) $\sigma:\sR^{m\times m+n}\mapsto\sR^{m\times m'}$ producing $\mV\in\R^{m\times m'}$, and the directional output is acquired by taking a matrix multiplication of the basis $\vec{\mZ}$ with $\mV$. By joining the analyses in GMN along with~\cite{villar2021scalars}, we have the universality of the formulation in Eq.~(\ref{eq:GMN}), which is provided in Appendix~\ref{sec:proof}. 

\begin{figure}[t]
    \centering
    \includegraphics[width=0.98\textwidth]{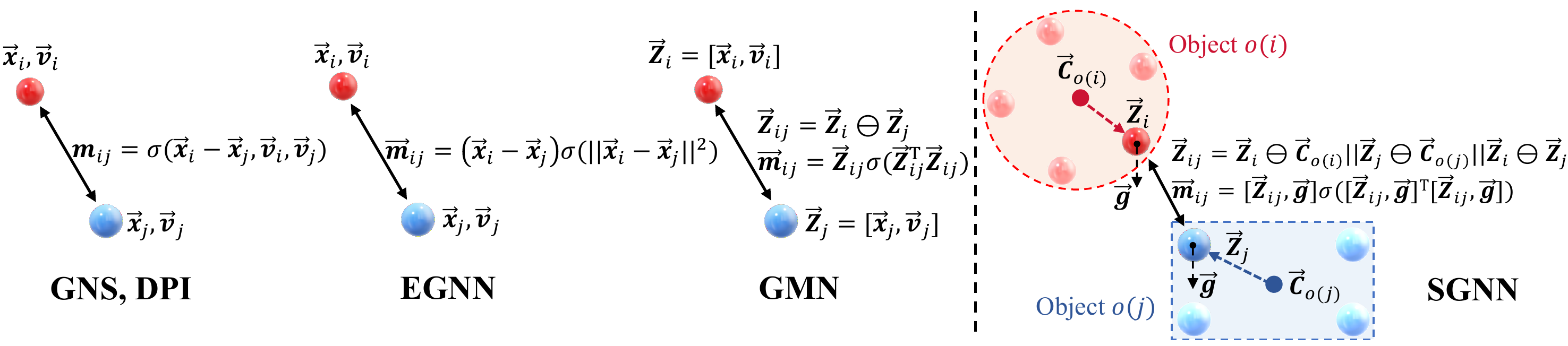}
    \caption{Comparison between different message passing. We omit the scalars $\vh_i$ for simplicity.}
    \label{fig:messagepassing}
\end{figure}

\subsection{Subequivariant Object-aware Message Passing}
\label{sec:SOMP}

To simultaneously benefit from symmetry and object-aware information, we develop Subequivariant Object-aware Message Passing (SOMP). It characterizes information aggregation between multi-channel vectors and scalars, while taking into account both the particle and object information. We use $\vec{\mZ}_i\in\sR^{3\times m}$ to indicate a stack of $m$ multi-channel 3D vectors; particularly, $\vec{\mZ}_i$ is initialized as $[\vec{\vx}_i, \vec{\vv}_i], 1\leq i\leq N$ by involving both the position and velocity. We additionally initialize the object features by pooling its particles: for object $k$, we have $\vec{\mC}_k=\frac{1}{|\{i:o(i)=k\}|}\sum\nolimits_{\{i:o(i)=k\}}\vec{\mZ}_i, \vc_k=\sum\nolimits_{\{i:o(i)=k\}}\vh_i,
1\leq k\leq M$.
We also introduce a binary operation ``$\ominus$'' which yields $\vec{\mZ}_i\ominus\vec{\mZ}_j = [\vec{\mZ}_i-\vec{\mZ}_j, \vec{\vv}_i, \vec{\vv}_j]$. It transforms the translation-equivariant vectors to be invariant and stacks those invariant vectors together, resulting in a translation-invariant representation while enriching the input vectors with more channels. The direction of gravity is set to be along the vertical axis.


In a high-level overview, we denote our message passing as the function $\varphi$ that updates each particle given the input states of all particles, objects, and graph connectivity:
\begin{align}
\label{eq:somp}
    \{(\vec{\mZ}'_i, \vh'_i)\}_{i=1}^N &= \varphi\left(\{(\vec{\mZ}_i,\vh_i)\}_{i=1}^N, \{(\vec{\mC}_{k}, \vc_{k})\}_{k=1}^M, \gE \right).
\end{align}
Specifically, $\varphi$ is unfolded as the following message passing and aggregation computations:
\begin{align}
\label{eq:vector-interaction}
\vec{\mZ}_{ij} &= (\vec{\mZ}_i\ominus\vec{\mC}_{o(i)}) \|(\vec{\mZ}_j\ominus\vec{\mC}_{o(j)})\| (\vec{\mZ}_i\ominus\vec{\mZ}_j),\\
\label{eq:scalar-interaction}
\vh_{ij} &=  \vh_i\| \vc_{o(i)} \| \vh_j \| \vc_{o(j)}, \\
\label{eq:message-mlp}
\vec{\mM}_{ij}, \vm_{ij} &= \phi_{\vec{\vg}}\left(\vec{\mZ}_{ij}, \vh_{ij} \right), \\
\label{eq:update-mlp}
(\vec{\mZ}'_i, \vh'_i) &= (\vec{\mZ}_i, \vh_i) + \psi_{\vec{\vg}}\left((\sum\nolimits_{j\in\gN(i)}\vec{\mM}_{ij}) \| (\vec{\mZ}_i\ominus\vec{\mC}_{o(i)}), (\sum\nolimits_{j\in\gN(i)}\vm_{ij}) \| \vh_i \| \vc_{o(i)} \right), 
\end{align}
where $\|$ is the concatenation along the channel dimension, $\gN(i)=\{j: (i,j)\in\gE \}$ is the neighbors of node $i$, and both $\phi_{\vec{\vg}}$ and $\psi_{\vec{\vg}}$ are subequivariant and will be defined in detail in Eq.~(\ref{eq:subGMN}).
In detail, we first derive the multi-channel vector input $\vec{\mZ}_{ij}$ for the interaction between particle $i$ and $j$ in Eq.~(\ref{eq:vector-interaction}) by stacking the three terms along the channel dimension, including the relative information of the particle with respect to its belonged object: $\vec{\mZ}_i\ominus\vec{\mC}_{o(i)}$ and $\vec{\mZ}_i\ominus\vec{\mC}_{o(i)}$, and the relative information between particles $\vec{\mZ}_i\ominus\vec{\mZ}_j$. By this means, we arrive at a translation-invariant representation $\vec{\mZ}_{ij}$ with rich object-aware geometric information. For the invariant features we simply concatenate them in Eq.~(\ref{eq:scalar-interaction}).
Subsequently, in Eq.~(\ref{eq:message-mlp}) we feed the interactions $\vec{\mZ}_{ij}$ and $\vh_{ij}$ into the subequivariant message function $\phi_{\vec{\vg}}$, yielding the vector message $\vec{\mM}_{ij}$ and scalar message $\vm_{ij}$\footnote{For both $\phi_{\vec{\vg}}$ and $\psi_{\vec{\vg}}$, we expand more output channels besides $\mV_{\vec{\vg}}$ of $\sigma$ in Eq.~(\ref{eq:subGMN}), and assign it to the invariant message $\vm_{ij}$ in Eq.~(\ref{eq:message-mlp}) and $\vh'_i$ in Eq.~(\ref{eq:update-mlp}), respectively.}. Finally, Eq.~(\ref{eq:update-mlp}) first performs message aggregation and then updates the states by another subequivariant function $\psi_{\vec{\vg}}$, obtaining the eventual result. With the updated $\vec{\mZ}'_i$ and $\vh'_i$, one can readily obtain the output with equivariance or invariance as desired, which, in our case, implies $\vec{\mZ}'_i= [\vec{\vx}'_i]$ by setting $m'=1$.

\textbf{Comparison with existing GNN-based simulators.} The core of satisfying equivariance lies in taking the invariant inner product before feeding the vectors into the MLP. Such design shares a similar spirit in nature with existing works including EGNN~\cite{satorras2021en} and GMN~\cite{huang2022equivariant}. The interaction modeling of EGNN only considers the relative position $\vec{\vx}_i-\vec{\vx}_j$ and its inner product $\|\vec{\vx}_i-\vec{\vx}_j\|^2$, which might have limited expressivity while tackling complex physical interactions and dynamics between objects. GMN further extends to a multi-channel interaction $\vec{\mZ}_i\ominus\vec{\mZ}_j$ where $\vec{\mZ}_i=[\vec{\vx}_i,\vec{\vv}_i]$. Nevertheless, it lacks the necessary object information, while we explicitly involve the particle-object correlation $\vec{\mZ}_i\ominus\vec{\mC}_{o(i)}$ in our SOMP.
Besides, GNS~\cite{sanchez2020learning} and DPI~\cite{li2018learning} only enforces translation equivariance by directly feeding $[\vec{\vx}_i-\vec{\vx}_j,\vec{\vv}_i,\vec{\vv}_j]$ into an MLP without taking the inner product. They fail to preserve the O($3$)-equivariance and consequently have weaker generalization as we will illustrate in our experiments.
Fig.~\ref{fig:messagepassing} summarizes and compares the message-passing schemes of GNS, EGNN, GMN and our SOMP. Particularly, our design of SOMP takes careful considerations: \textbf{1.} Equivariance is still permitted with both geometric and scalar object features involved; \textbf{2.} The object information can also be constantly updated during message passing; \textbf{3.} The expressivity of SOMP is enhanced over EGNN and GMN with object information considered.
We provide detailed theoretical comparisons in Appendix~\ref{sec:proof2} by showing that EGNN and GMN are indeed special cases of SGNN.


We now present the detailed formulations of $\phi_{\vec{\vg}}$ in Eq.~(\ref{eq:message-mlp})  and $\psi_{\vec{\vg}}$ in Eq.~(\ref{eq:update-mlp}).  Since the full symmetry is violated by gravity $\vec{\vg}\in\R^3$, and the dynamics of the system will naturally preserve a gravitational acceleration in the vertical direction. By this means, the orthogonal symmetry is no longer maintained in every direction but only restricted to the subgroup $O_{\vec{\vg}}(3)\coloneqq\{\mO\in O(3)\mid \mO\vec{\vg}=\vec{\vg}\}$, that is, the rotations/reflections around the gravitational axis. We term such a reduction of equivariance as a novel notion: \emph{subequivariance}. To reflect this special symmetry, we augment Eq.~(\ref{eq:GMN}) by:
\begin{align}
\label{eq:subGMN}
    f_{\vec{\vg}}(\Vec{\mZ},\vh)=[\Vec{\mZ}, \vec{\vg}]\mV_{\vec{\vg}}, \quad\text{s.t.} \mV_{\vec{\vg}}=    \sigma([\Vec{\mZ},\vec{\vg}]^{\top}[\Vec{\mZ},\vec{\vg}],\vh),
\end{align}
where $\sigma:\R^{(m+1)\times(m+1)}\rightarrow\R^{(m+1)\times m'}$ is an MLP. Compared with Eq.~(\ref{eq:GMN}), here we just augment the directional input with $\vec{\vg}$. Interestingly, such a simple augmentation is universally expressive:
\begin{theorem}
\label{th:subequ}
Let $f_{\vec{\vg}}(\Vec{\mZ},\vh)$ be defined by Eq.~(\ref{eq:subGMN}). Then, $f_{\vec{\vg}}$ is $O_{\vec{\vg}}(3)$-equivariant. More importantly, For any $O_{\vec{\vg}}(3)$-equivariant function $\hat{f}(\Vec{\mZ},\vh)$, there always exists an MLP $\sigma$ satisfying $\|\hat{f}-f_{\vec{\vg}}\|<\epsilon$ for arbitrarily small positive value $\epsilon$.
\end{theorem} 
The proof is non-straightforward and deferred to Appendix~\ref{sec:proof1}. The detailed architectural view of the proposed SOMP is depicted in Fig.~\ref{fig:theo-compare} in Appendix~\ref{sec:proof3}.
We leverage our designed subequivariant function with such desirable properties for $\phi_{\vec{\vg}}$ and $\psi_{\vec{\vg}}$ in our subequivariant object-aware message passing. We immediately have the following theorem guaranteeing the validity of our design.
\begin{theorem}
The message passing $\varphi$ (Eq.~\ref{eq:somp}) is $O_{\vec{\vg}}(3)$-equivariant.
\end{theorem}

\subsection{Application to Physical Scenes with Hierarchical Modeling}
\label{sec:hier}

This subsection introduce the entire architecture of our subequivariant GNN. Many physical scenes are complicated, possibly involving contact, collision, and friction amongst multiple objects. In light of this, we further incorporate our message passing into a multi-stage hierarchical modeling framework. One of our interesting findings here is the edge separation. This is indeed motivated by the consideration that the interactions \emph{between} or \emph{within} objects are usually different, similar to automorphism graph networks~\cite{de2020natural,mitton2021local,thiede2021autobahn} where the neighboring objects formulate the isomorphism group of edges. The former serves as a bridge of exchanging momentum and energy via interaction forces, while the latter usually accounts for maintaining the rigid or plastic constraints. Therefore, it would be beneficial to disentangle these two kinds of interactions in message passing, which also distinguishes our hierarchical modeling from DPI~\cite{li2018learning}. The overall flowchart is provided in Fig.~\ref{fig:architecture}.
To be specific, we alternate between the following three stages that implement $\varphi$ in Eq.~(\ref{eq:somp}) distinctly. 

\textbf{1. Particle-level inter-object message passing.} We start by modeling the local interactions from the particle-level, where we only involve those between different objects. That is,
\begin{align}
\label{eq:particle-level-inter}
    \{(\vec{\mZ}'_i, \vh'_i)\}_{i=1}^N &= \varphi_1\left(\{(\vec{\mZ}_i,\vh_i)\}_{i=1}^N, \{(\vec{\mC}_{k}, \vc_{k})\}_{k=1}^M, \gE_{\text{inter}} \right),
\end{align}
where $\gE_{\text{inter}} = \{(i,j): (i,j)\in\gE, o(i)\neq o(j)\}$.

\textbf{2. Object-level message passing.} Given the renewed particle-level information produced by the first stage, we are now ready to construct the object-level message passing graph. The object-level message passing is proceeded as
\begin{align}
    \{(\vec{\mC}'_k, \vc'_k)\}_{k=1}^M &= \varphi_2\left(\{(\vec{\mC}_k,\vc_k)\}_{k=1}^M, \{\emptyset\}, \gE_{\text{obj}} \right),
\end{align}
where $\gE_{\text{obj}}$ comprises the edges between interacted objects, indicating $\gE_{\text{obj}} = \{(k,l): \exists (i,j)\in\gE_{\text{inter}}, o(i)=k, o(j)=l \}$. As a specific instantiation of $\varphi$ in Eq.~(\ref{eq:somp}), here each object is considered as a node and the original object information in Eq.~(\ref{eq:vector-interaction}-\ref{eq:update-mlp}) is omitted (this is why we denote the second input of $\varphi_2$ as the empty set $\emptyset$). Furthermore, we leverage a pooling of the updated particle-level information from the first stage as the object-level interactions, or formally, we re-define the calculation of ``$\ominus$'' for two different objects indexed by $k, l$ as $\vec{\mC}_k\ominus\vec{\mC}_l = \text{Mean-Pool}_{(i,j)\in\gE_{\text{inter}}}\left(\vec{\mZ}'_i\ominus\vec{\mZ}'_j \right)$, and similarly, $\vc_k\| \vc_l = \text{Mean-Pool}_{(i,j)\in\gE_{\text{inter}}}\left( \vh'_i\| \vh'_j \right)$. 

\textbf{3. Particle-level inner-object message passing.} Finally, given the updated object state information, we carry out the message passing for the particles only along the inner-object edges $\gE_{\text{inner}}=\{(i,j):(i,j)\in\gE, o(i)=o(j)\}$, namely,
\begin{align}
    \{(\vec{\mZ}''_i, \vh''_i)\}_{i=1}^N &= \varphi_3\left(\{(\vec{\mZ}_i,\vh_i)\}_{i=1}^N, \{(\vec{\mC}'_{k}, \vc'_{k})\}_{k=1}^M, \gE_{\text{inner}} \right).
\end{align}
\begin{figure}[t]
    \centering
    \includegraphics[width=0.95\textwidth]{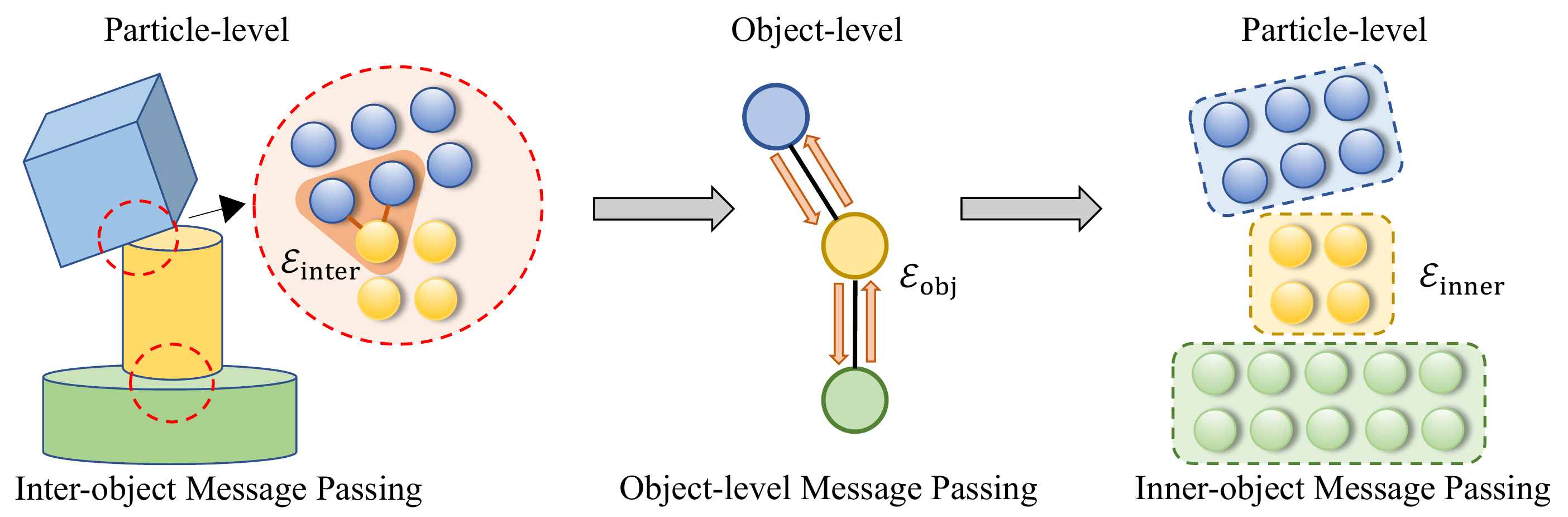}
    \caption{The overall architecture with hierarchical modeling, specified as three stages.}
    \label{fig:architecture}
\end{figure}
We use $\{\vec{\mZ}''_i\}_{i=1}^N$ as the proposal of the positions in the next time step, akin to the way in~\textsection~\ref{sec:SOMP}.

\section{Experiments}
\label{sec:exp}
We conduct evaluations on Physion~\cite{bear2021physion} and RigidFall~\cite{li2020visual}. Physion~\cite{bear2021physion} is a large scale dataset created by the ThreeDWorld simulator~\cite{gan2021threedworld}, which consists of eight different scenarios, including Dominoes, Contain, Collide, Drop, Roll, Link, Support, and Drape. These scenarios possess diverse physical scenes with complicated object interactions, serving as a challenging benchmark for evaluating dynamics models. Particularly, Drape also involves simulating \emph{deformable} objects. RigidFall~\cite{li2020visual} is a simulation dataset whose scenes involve several cubes falling and colliding under varying magnitude of gravitational acceleration.

\subsection{Baselines} 
We compare SGNN with several SOTA particle-based GNN simulators, including two non-equivariant models GNS~\cite{sanchez2020learning} and DPI~\cite{li2018learning}, as well as two equivariant models EGNN~\cite{satorras2021en} and GMN~\cite{huang2022equivariant}. Notably, the data preprocessing provided by Physion~\cite{bear2021physion} for particle-based methods does not consider the camera-angle for the rotation of 3D coordinates, leading to the situation that in some scenarios, \emph{e.g.}, Dominoes, Link, Support, and Collide, there exist a directional bias (from left to right horizontally) in the trajectories. This motivates us to consider three evaluation protocols: \textbf{1.} Use the original training and testing set with directional bias. \textbf{2.} Train on the original set and test on the randomly-rotated testing set. \textbf{3.} Train on the randomly-rotated training set (akin to applying data augmentation of rotations~\cite{benton2020learning}), and test on the randomly-rotated testing set. Here all rotations are restricted to the ones around gravity. We dub the results obtained from the three cases GNS*, GNS, and GNS-Rot for GNS, and similarly for DPI. For EGNN and GMN, we also propose to adapt them to be subequivariant (EGNN-S and GMN-S) by adding an extra vector along the gravity in their update of velocities. Details are in Appendix~\ref{sec:impl-detail}. Note that EGNNs, GMNs, and our SGNN always produce exactly the same result in the three scenarios due to their equivariance (or subequivariance). For RigidFall, such bias is not observed since the cubes are falling vertically driven by gravity.

\begin{table}[t]
  \centering
  \small
   \setlength{\tabcolsep}{4.5pt}
  \caption{Contact prediction accuracy (\%) on Physion. Results are averaged across 3 runs.}
    \begin{tabular}{lcccccccc}
    \toprule
    & Dominoes & Contain & Link  & Drape & Support & Drop  & Collide & Roll \\
    \midrule
    GNS*~\cite{sanchez2020learning}  & 78.6\tiny{$\pm$0.9} & 71.6\tiny{$\pm$1.6} & 66.7\tiny{$\pm$1.5} & 58.8\tiny{$\pm$1.0} & 68.2\tiny{$\pm$1.6} & 65.3\tiny{$\pm$1.1} & \textbf{86.1}\tiny{$\pm$0.5} & 81.3\tiny{$\pm$1.8} \\
    DPI*~\cite{li2018learning}  & 82.3\tiny{$\pm$1.3} & 72.3\tiny{$\pm$1.8} & 63.7\tiny{$\pm$2.2} & 53.3\tiny{$\pm$0.9} & 64.8\tiny{$\pm$2.0} & 70.7\tiny{$\pm$0.8} & 84.4\tiny{$\pm$0.7} & 82.3\tiny{$\pm$0.6} \\
    \midrule
    GNS~\cite{sanchez2020learning}   & 53.3\tiny{$\pm$1.8} & 70.7\tiny{$\pm$2.2} & 58.0\tiny{$\pm$1.7} & 58.0\tiny{$\pm$1.3} & 61.3\tiny{$\pm$1.7} & 65.0\tiny{$\pm$0.7} & 76.7\tiny{$\pm$1.2} & 80.0\tiny{$\pm$0.4} \\
    DPI~\cite{li2018learning}   & 57.3\tiny{$\pm$2.2} & 71.9\tiny{$\pm$1.3} & 63.0\tiny{$\pm$1.9} & 52.3\tiny{$\pm$1.2} & 58.0\tiny{$\pm$1.1} & 68.7\tiny{$\pm$0.7} & 78.7\tiny{$\pm$1.3} & 80.5\tiny{$\pm$0.8} \\
    \midrule
    GNS-Rot~\cite{sanchez2020learning} & 74.7\tiny{$\pm$1.2} & 72.7\tiny{$\pm$2.0} & 63.1\tiny{$\pm$2.1} & 57.3\tiny{$\pm$1.0} & 65.5\tiny{$\pm$1.3} & 64.7\tiny{$\pm$1.3} & 84.0\tiny{$\pm$1.0} & 79.4\tiny{$\pm$2.1} \\
    DPI-Rot~\cite{li2018learning} & 72.7\tiny{$\pm$1.7} & 69.4\tiny{$\pm$2.1} & 65.2\tiny{$\pm$2.3} & 52.8\tiny{$\pm$1.0} & 66.7\tiny{$\pm$0.8} & 72.3\tiny{$\pm$0.5} & 83.2\tiny{$\pm$1.6} & 80.1\tiny{$\pm$0.5} \\
    \midrule
    EGNN~\cite{satorras2021en}  & 61.3\tiny{$\pm$1.1} & 66.0\tiny{$\pm$1.9} & 52.7\tiny{$\pm$1.2} & 54.7\tiny{$\pm$0.8} & 60.0\tiny{$\pm$0.5} & 63.3\tiny{$\pm$0.9} & 76.7\tiny{$\pm$1.4} & 79.8\tiny{$\pm$0.4} \\
     EGNN-S  & 72.0\tiny{$\pm$2.1} & 64.6\tiny{$\pm$1.8} & 55.3\tiny{$\pm$2.1} & 55.3\tiny{$\pm$1.0} & 60.5\tiny{$\pm$1.5} & 69.3\tiny{$\pm$2.0} & 79.3\tiny{$\pm$1.5} & 81.6\tiny{$\pm$0.9} \\
    GMN~\cite{huang2022equivariant}  & 54.7\tiny{$\pm$0.5} & 57.6\tiny{$\pm$2.2} & 54.5\tiny{$\pm$1.6} & 57.6\tiny{$\pm$1.3} & 55.1\tiny{$\pm$1.6} & 54.2\tiny{$\pm$1.1} & 79.5\tiny{$\pm$1.5} & 81.3\tiny{$\pm$1.3} \\
        GMN-S  & 55.6\tiny{$\pm$1.3} & 65.3\tiny{$\pm$1.7} & 55.1\tiny{$\pm$2.8} & 57.0\tiny{$\pm$1.8} & 59.3\tiny{$\pm$2.2} & 57.3\tiny{$\pm$1.5} & 81.2\tiny{$\pm$1.2} & 79.3\tiny{$\pm$1.5} \\
    \midrule
    Our SGNN  & \textbf{89.1}\tiny{$\pm$1.5} & \textbf{78.1}\tiny{$\pm$1.5} & \textbf{73.3}\tiny{$\pm$1.1} & \textbf{60.6}\tiny{$\pm$0.5} & \textbf{71.2}\tiny{$\pm$0.9} & \textbf{74.3}\tiny{$\pm$1.0} & 85.3\tiny{$\pm$1.1} & \textbf{84.2}\tiny{$\pm$0.6} \\
    \bottomrule
    \end{tabular}%
  \label{tab:physion_acc}%
\end{table}%

\begin{figure}
    \centering
    \includegraphics[width=0.99\linewidth]{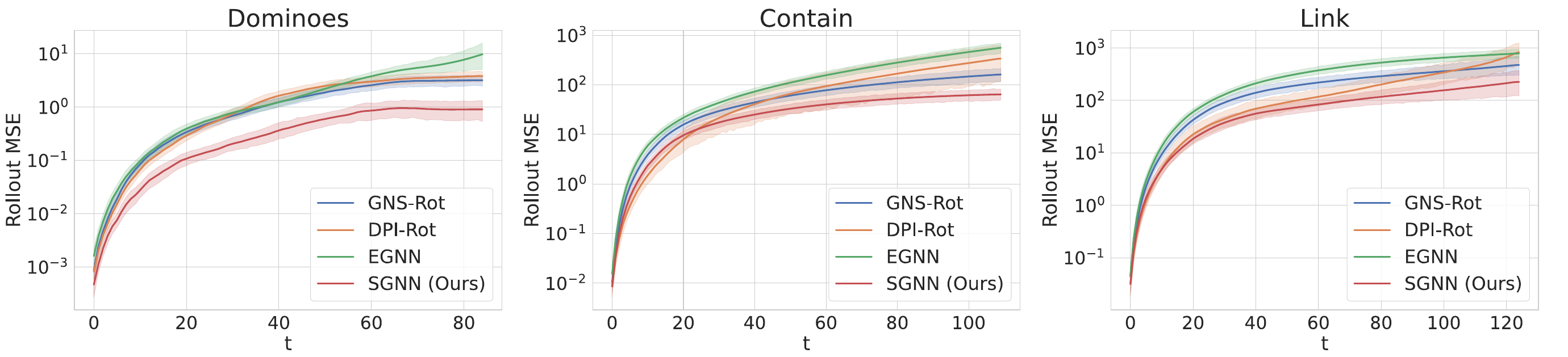}
    \caption{The rollout-MSE curves on Dominoes, Link, and Contain. Our model achieves the lowest MSE. Colored areas indicate 95\% confidence intervals. Full results in Appendix~\ref{sec:mse}.}
    \label{fig:rollout_mse}
    \vskip -0.2in
\end{figure}


\subsection{Evaluation on Physion}

\textbf{Experimental setup.} We strictly follow the training and evaluation protocol proposed in~\cite{bear2021physion} for the particle-based simulators. We employ two evaluation metrics: \textbf{1.} Contact prediction accuracy. This metric, as also adopted by~\cite{bear2021physion}, evaluates whether two targeted objects labeled in the dataset will contact during the model rollout. \textbf{2.} Rollout MSE, the mean squared error between model rollout and the ground truth. We reuse the training protocol and hyper-parameters suggested in Physion for baselines and our model to ensure a fair comparison. In detail, all MLPs are initialized with 3 projection layers and a hidden dimension of 200. The networks are trained with an Adam optimizer, using an initial learning rate 0.0001 and an early-stopping of 10 epochs on the validation loss. We use 4 iterations in each message passing of our model. 
We run RANSAC~\cite{fischler1981random} during testing on all models to help enforce the rigid constraint as suggested by Physion. More experimental details including a comparison of computational complexity are deferred to Appendix~\ref{sec:impl-detail}.

\textbf{Results.} We present the average contact prediction accuracy with std in Table~\ref{tab:physion_acc}, and the rollout MSE in Fig.~\ref{fig:rollout_mse}. The detailed MSE on all 8 scenarios is deferred to Appendix~\ref{sec:mse} due to space limit. We interpret our results by answering the following questions. \textbf{Q1.} \emph{Is maintaining the symmetry important to models that learn to simulate?} For the non-equivariant models GNS and DPI, their performance usually encounters a significant drop if they are trained on the original data but tested on the rotated, particularly on the scenarios with bias in directions. For example, the accuracy of GNS drops from 78.6\% to 53.3\% on Dominoes when comparing GNS* with GNS. Such observation also holds for DPI and on more scenarios like Collide and Support. Using data augmentation would help in relieving the issue, but is still hard to recover the performance. The accuracy on Dominoes yielded by GNS-Rot is 74.7\%, still worse than 78.6\% by GNS*. Our SGNN instead leverages subequivariance to generalize to all horizontal directions, and is always guaranteed to produce the same prediction regardless of any rotations around gravity, and there is also no need to apply any data augmentation of rotations during training. Interestingly, EGNN-S and GMN-S offers improvements over EGNN and GMN, implying the necessity of subequivariance, but are still outperformed by SGNN by a large margin. \textbf{Q2.} \emph{Does SGNN learn the physics and simulate the dynamics better than other models?} As displayed in Table~\ref{tab:physion_acc}, SGNN achieves the highest prediction accuracy on 7 out of the 8 scenarios, and is also very competitive on Collide. The improvements are very significant (>5\% accuracy enhancement over the second best) on scenarios including Dominoes, Contain, and Link, where there are multiple objects contacting, colliding, and interacting. This observation is also supported by the rollout MSE curves in Fig.~\ref{fig:rollout_mse}, showing our SGNN consistently yielding lower prediction error along the trajectories. \textbf{Q3.} \emph{How does SGNN perform compared with other equivariant models?} For EGNN and GMN, enforcing such a strong constraint of E($3$)-equivariance leads to inferior performance in scenes with gravity. For instance, in Dominoes, Contain, and Link, the accuracy of EGNN and GMN only stays around 60\%, while SGNN achieves promising results with subequivariant object-aware message passing.
\begin{figure}[t]
    \centering
    \includegraphics[width=0.95\textwidth]{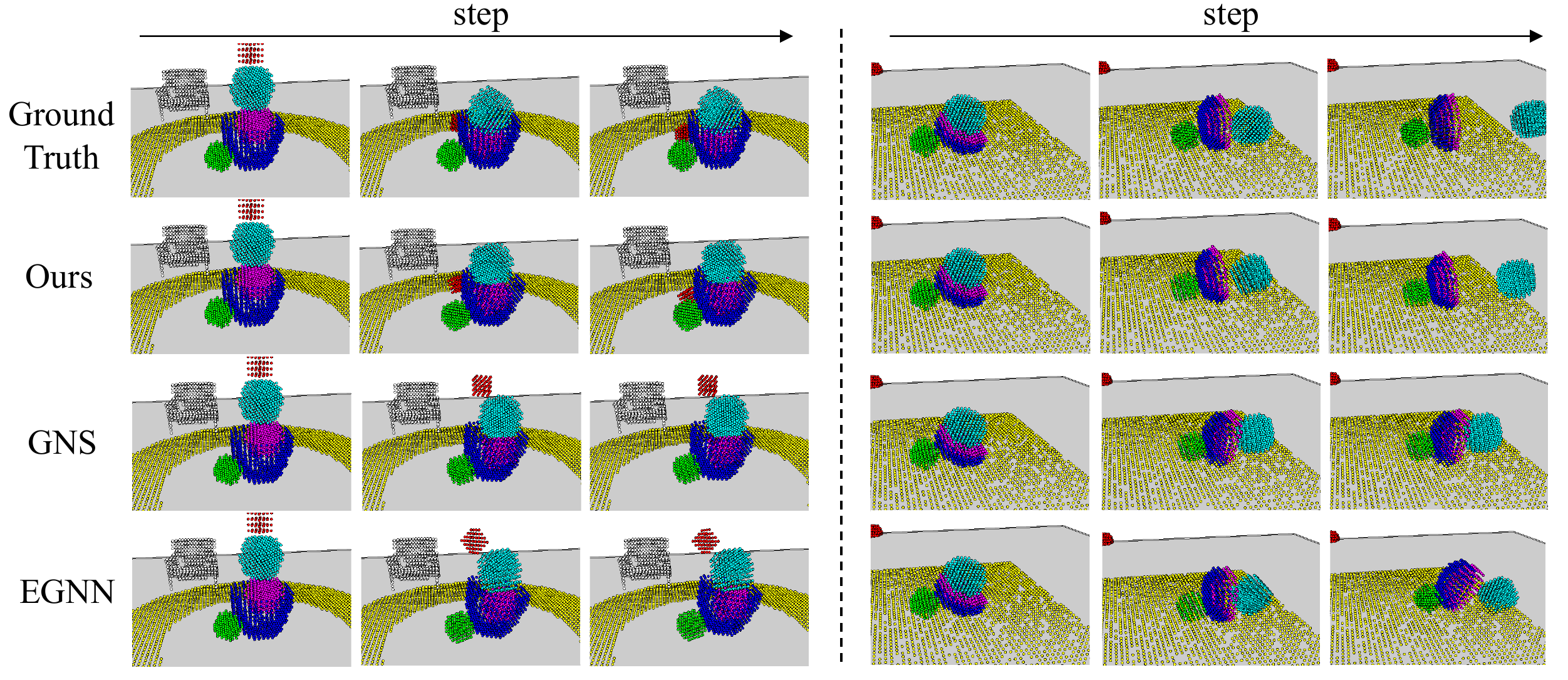}
    \caption{Visualization examples on Physion. Our model yields accurate long-term predictions.}
    \label{fig:physion_vis}
\end{figure}

\textbf{Generalization across scenarios.} We further evaluate the generalization capability of our model across various scenarios. To fulfill this goal, we leverage models trained in certain scenarios
\begin{wrapfigure}[13]{r}{0.55\textwidth}
  \begin{center}
  \vskip -0.20in
    \includegraphics[width=0.55\textwidth]{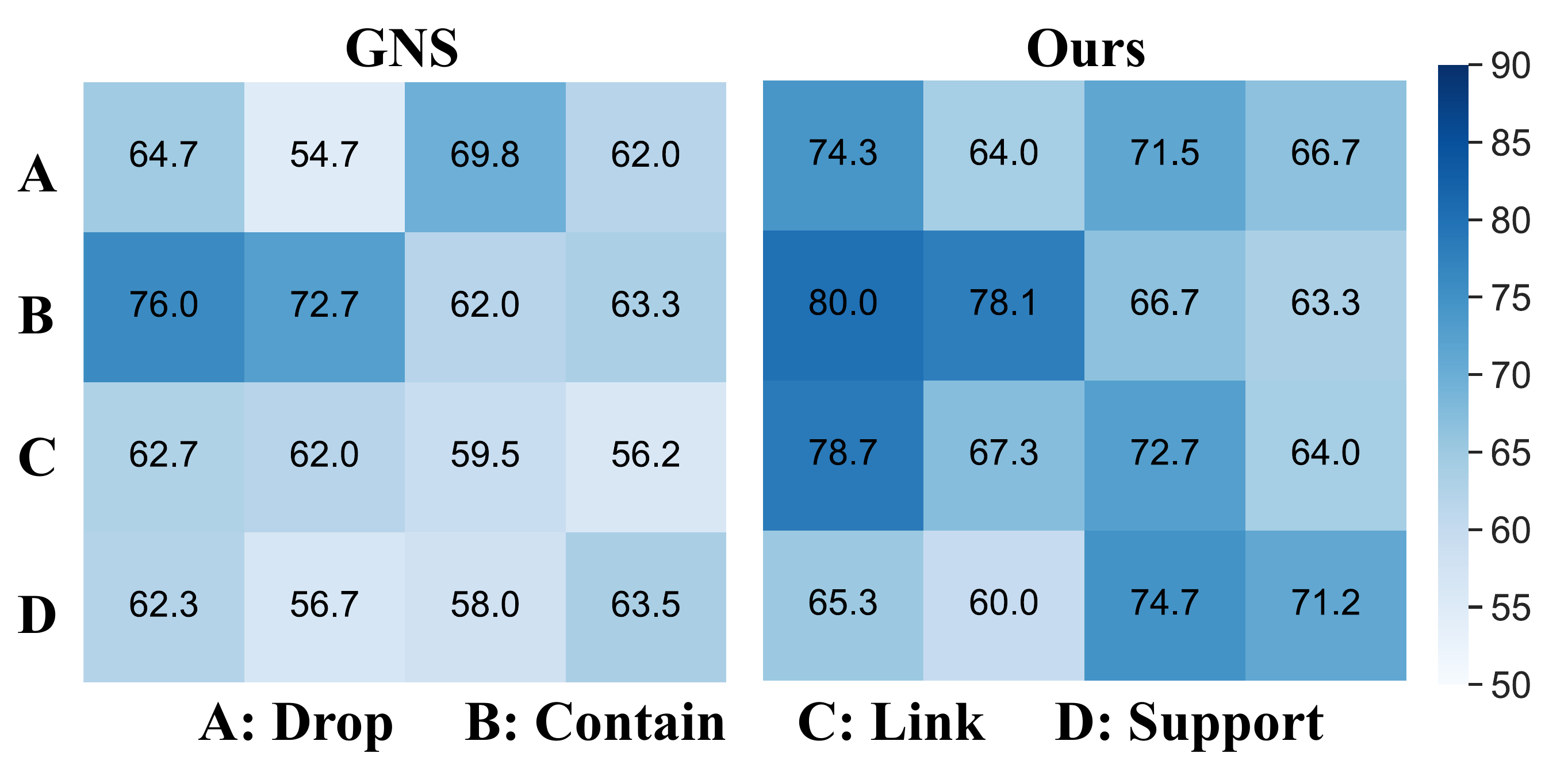}
  \end{center}
  \vskip-0.15in
  \caption{Generalization across tasks. Row/column indicates the training/testing scenario, respectively.}
  \label{fig:generalization}
\end{wrapfigure}
to test them in others. We summarize the results in Fig.~\ref{fig:generalization}, where the models are trained in the scenarios indexed by the rows and evaluated on those indexed by the columns. It is observed that our model exhibits significantly stronger generalization than GNS, the best-performed baseline on Physion. For example, our model, although trained with the dynamics of Contain or Link, still yields a high accuracy (>78\%) when tested on Drop, showcasing that our model is more advantageous in learning the dynamics and interactions of various objects.


\begin{wraptable}{r}{0.5\textwidth}
  \centering
  \small
    \setlength{\tabcolsep}{3.2pt}
  \vskip -0.15in
  \caption{Ablation studies.}
    \begin{tabular}{lccc}
    \toprule
          & Dominoes & Contain & Drop \\
    \midrule
    SGNN  &  89.1     & 78.1 & 74.1 \\
    \midrule
    E($3$)-equivariance & 79.6      & 65.3 & 68.1\\
    w/o object-aware & 81.2      & 68.7 & 70.1\\
    w/o hierarchy & 83.5      & 72.5 & 71.3\\
    w/o edge separation &  85.0     & 74.8  & 72.0\\
    \midrule
    Steer-SE(2)-GNN~\cite{weiler2019general}  &  59.1 & 66.7 & 66.7 \\
    \bottomrule
    \end{tabular}%
  \label{tab:abstudy}%
\end{wraptable}
\textbf{Ablation studies.} Here we conduct several ablations to inspect how our proposed components contribute to the overall performance, and the results are given in Table~\ref{tab:abstudy}.
\textbf{1.} Subequivariance. We replace our subequivariant formulation of $\phi_{\vec{\vg}}$ and $\psi_{\vec{\vg}}$ by the E($3$)-equivariant counterpart without gravity. It is clear that, in this way, the model is restricted by strong equivariance constraints, which leads to a significant degradation in performance. This verifies the necessity of relaxing the full-equivariant model with our introduced method for physical scenarios with gravity. Moreover, we extend the Steerable E(2)-CNNs~\cite{weiler2019general,cesa2021program} to GNNs as an alternative to obtain equivariance on the SE(2) subgroup (more details are results are in Appendix~\ref{sec:se2gnn}). Although better than EGNN, the adapted Steerable SE(2)-GNN is still inferior to SGNN, implying our physics-inspired SOMP is more advantageous on simulating physical dynamics.
\textbf{2.} Object-aware message passing. We replace $\vec{\mC}_k$ and $\vc_k$ in Eq.~(\ref{eq:particle-level-inter}) by zeros, eliminating the object information from message passing. Without this necessary information, the model might fail to capture useful geometric vectors like $\vec{\vx}_i-\vec{\vx}_{o(i)}$, which is closely related to physical quantities like torque.
\textbf{3.} Hierarchy. We compare our model with its counterpart without hierarchy. The flat model encounters an average of 4\% drop in the prediction accuracy, showcasing the advantages of leveraging the particle- and object-level message passing for modeling complex object interactions in physical scenes.
\textbf{4.} Edge separation. We employ different sets of edges in the inter-object and inner-object message passing. We argue that it relieves the learning complexity of the message-passing with different types of relations. It is observed that removing edge separation leads to detriment in the accuracy in various scenarios.

\textbf{Generalization toward other rotations.} 
To further investigate whether incorporating our subequivariance truly helps the model to learn the effect of gravity, we apply a rotation around a \emph{non-gravity}
\begin{wrapfigure}[9]{r}{0.50\textwidth}
  \begin{center}
  \vskip -0.22in
    \includegraphics[width=0.50\textwidth]{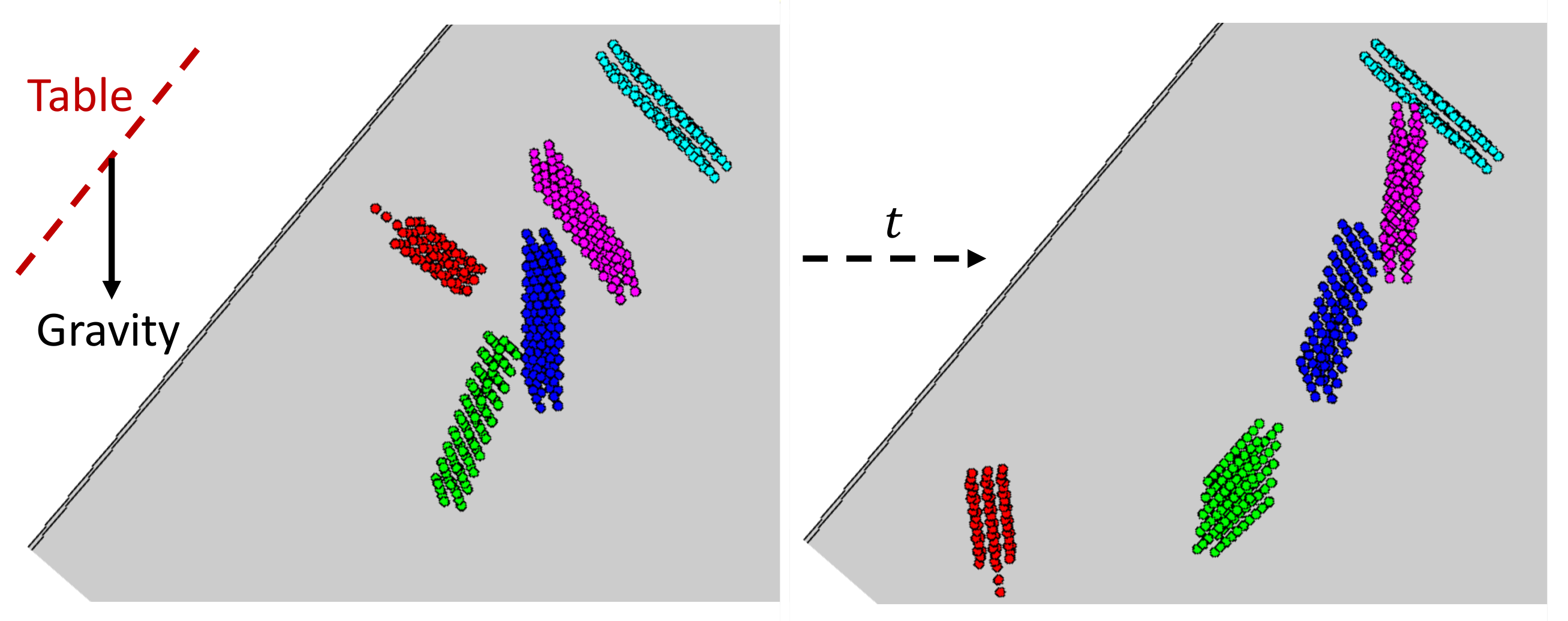}
  \end{center}
  \vskip-0.15in
  \caption{Rotation around a non-gravity axis.}
  \label{fig:rotation}
\end{wrapfigure}
axis, resulting in scenarios in Fig.~\ref{fig:rotation} where dominoes are placed on an incline while gravity still points downwards vertically. Interestingly, SGNN well generalizes to this novel scenario and reasonably simulates the effect of gravity. Particularly, the domino at the bottom starts to slide down along the table driven by gravity. More visualizations are provided in the \href{https://hanjq17.github.io/SGNN/}{video}.

\textbf{Qualitative comparisons.} We provide visualization samples on Physion in Fig.~\ref{fig:physion_vis}. More cases are shown by our \href{https://hanjq17.github.io/SGNN/}{Supplementary Video}. Our model consistently yields accurate predictions across various scenarios.

\subsection{Evaluation on RigidFall}

\begin{table}[htbp]
  \centering
  \small
     \setlength{\tabcolsep}{4.2pt}
  \caption{Rollout MSE ($\times 10^{-2}$) on RigidFall. $|\text{Train}|$ denotes the number of training samples.}
    \begin{tabular}{lcccccccc}
    \toprule
          & \multicolumn{2}{c}{$|\text{Train}|= 200$} & \multicolumn{2}{c}{$|\text{Train}|= 500$} & \multicolumn{2}{c}{$|\text{Train}|= 1000$} & \multicolumn{2}{c}{$|\text{Train}|= 5000$} \\
          & $t=20$  & $t=40$  & $t=20$  & $t=40$  & $t=20$  & $t=40$  & $t=20$  & $t=40$ \\
    \midrule
    GNS~\cite{sanchez2020learning}   &  2.40\tiny{$\pm$1.32}     &  6.79\tiny{$\pm$3.43}     &  1.89\tiny{$\pm$0.87}     &  5.41\tiny{$\pm$2.78}    &  1.09\tiny{$\pm$0.70}     &  3.38\tiny{$\pm$2.15}     &     0.65\tiny{$\pm$0.43}  &  2.38\tiny{$\pm$1.66}\\
    DPI~\cite{li2018learning}   &  1.71\tiny{$\pm$0.87}     &  5.37\tiny{$\pm$2.86}     &  1.48\tiny{$\pm$0.78}     &  4.67\tiny{$\pm$2.55}     &   1.24\tiny{$\pm$0.72}    &  3.97\tiny{$\pm$2.35}    &  0.52\tiny{$\pm$0.36}     &  2.40\tiny{$\pm$1.53}  \\
    EGNN~\cite{satorras2021en}  &  1.89\tiny{$\pm$1.34}     &  5.66\tiny{$\pm$2.81}     &  1.07\tiny{$\pm$0.71}     &  3.94\tiny{$\pm$2.20}     &  0.95\tiny{$\pm$0.44}     &  2.72\tiny{$\pm$1.75}     &   0.78\tiny{$\pm$0.45}    &  2.63\tiny{$\pm$1.72}  \\
    GMN~\cite{huang2022equivariant}   &  2.74\tiny{$\pm$1.24}     & 7.08\tiny{$\pm$3.37}     &  2.50\tiny{$\pm$0.98}     &  6.34\tiny{$\pm$3.11}     &  1.91\tiny{$\pm$1.10}     &   4.50\tiny{$\pm$2.21}   & 1.36\tiny{$\pm$0.63}       & 3.50\tiny{$\pm$1.98}   \\
    \midrule
    Our SGNN  &  \textbf{0.72}\tiny{$\pm$0.60}     & \textbf{2.44}\tiny{$\pm$1.74}      &  \textbf{0.39}\tiny{$\pm$0.23}     &  \textbf{1.47}\tiny{$\pm$1.26}    &   \textbf{0.31}\tiny{$\pm$0.17}    &  \textbf{1.09}\tiny{$\pm$0.93}    &  \textbf{0.20}\tiny{$\pm$0.14}   & \textbf{1.01}\tiny{$\pm$0.97}   \\
    \bottomrule
    \end{tabular}%
  \label{tab:rigidfall_mse}%
\end{table}%

\textbf{Experimental setup.} We use the code provided by~\cite{li2018learning}. We record the rollout MSE at $t=20$ and $40$. To compare the data efficiency of different models, we consider four cases with the size of training set ranging from 200 to 5000. More details are in Appendix~\ref{sec:impl-detail}.


\textbf{Results.} The results are presented in Table~\ref{tab:rigidfall_mse} and visualized in Fig.~\ref{fig:rigidfall_vis}. Our model consistently yields the best predictions regardless of the time step $t$ and the size of training set. Notably, the rollout error still maintains very low even if only 200 training samples are provided, verifying the strong data-efficiency and generalization of SGNN.

\begin{figure}[t]
    \centering
    \vskip -0.15in
    \includegraphics[width=0.90\textwidth]{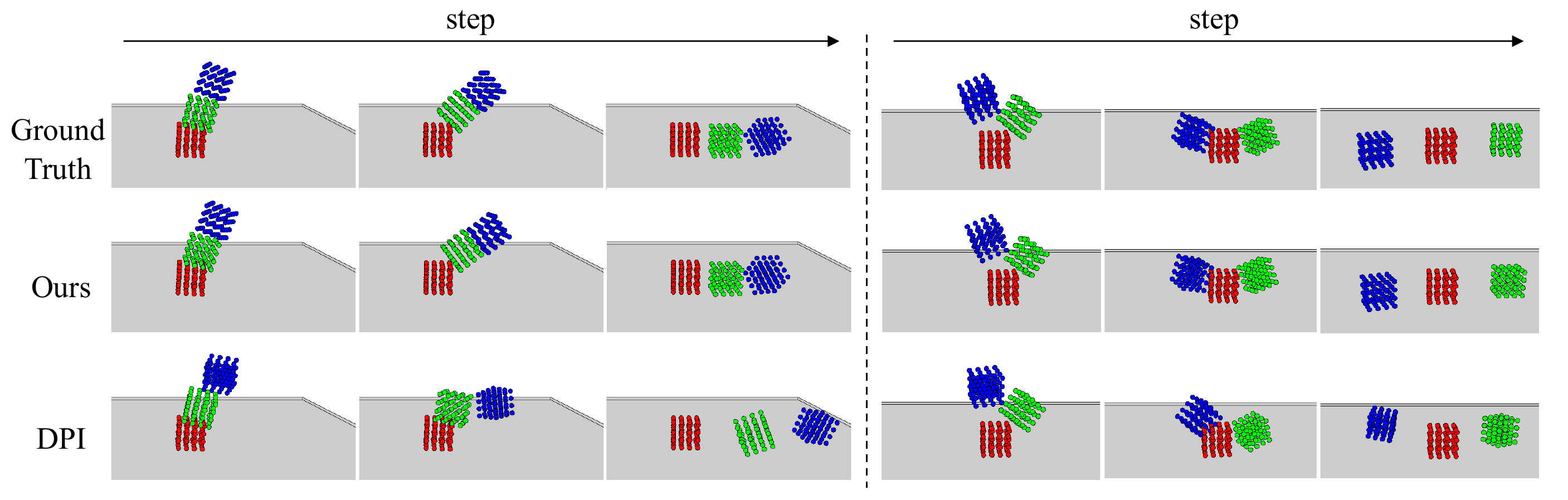}
    \vskip -0.1in
    \caption{Visualization examples on RigidFall. Our SGNN provides accurate simulation trajectories.}
    \label{fig:rigidfall_vis}
\end{figure}


\section{Discussion}
\label{sec:discussion}
\textbf{Limitations and Future Work.} Our model relies on the particle-based representation of the physical systems, which, in real-world scenarios, might be difficult to obtain, or usually with noise. Our model does not explicitly consider self-contact. Future directions include augmenting our dynamics model with strong physical priors like Hamiltonian~\cite{sanchez2019hamiltonian,zhong2021extending}, or combining it with groundings of visual prior~\cite{li2020visual} with application to more sophisticated tasks like physical reasoning or robot manipulation.

\textbf{Conclusion.}
We propose Subequivariant Graph Neural Networks for modeling physical dynamics of multiple interacting objects. We inject appropriate symmetry into a hierarchical message passing framework, and takes into account both particle- and object-level state messages. In particular, subequivariance is a novel concept for characterizing the relaxation of equivariance from a group to its subgroup. We show that relaxation from full equivariance to subequivariance can be applied to systems with gravity involved.  Experiments verify that SGNN accurately captures the dynamics with the guarantee of the desirable symmetry, and exhibits strong generalization with high data-efficiency.

\begin{ack}
We thank the anonymous reviewers for their constructive suggestions. This project was in part supported by MIT-IBM Watson AI Lab, Amazon Research Award Mitsubishi Electric. Dr. Wenbing Huang was supported by the following projects: the National Natural Science Foundation of China (No.62006137); Guoqiang Research Institute General Project, Tsinghua University (No. 2021GQG1012); Beijing Outstanding Young Scientist Program (No. BJJWZYJH012019100020098). 
\end{ack}

\bibliographystyle{plain}
\bibliography{neurips_2022}

\begin{thebibliography}{10}

\bibitem{agrawal2016learning}
Pulkit Agrawal, Ashvin~V Nair, Pieter Abbeel, Jitendra Malik, and Sergey
  Levine.
\newblock Learning to poke by poking: Experiential learning of intuitive
  physics.
\newblock {\em Advances in neural information processing systems}, 29, 2016.

\bibitem{battaglia2016interaction}
Peter Battaglia, Razvan Pascanu, Matthew Lai, Danilo Jimenez~Rezende, and koray
  kavukcuoglu.
\newblock Interaction networks for learning about objects, relations and
  physics.
\newblock In {\em Advances in Neural Information Processing Systems},
  volume~29, 2016.

\bibitem{battaglia2013simulation}
Peter~W Battaglia, Jessica~B Hamrick, and Joshua~B Tenenbaum.
\newblock Simulation as an engine of physical scene understanding.
\newblock {\em Proceedings of the National Academy of Sciences},
  110(45):18327--18332, 2013.

\bibitem{bear2021physion}
Daniel Bear, Elias Wang, Damian Mrowca, Felix~Jedidja Binder, Hsiao-Yu Tung,
  RT~Pramod, Cameron Holdaway, Sirui Tao, Kevin~A. Smith, Fan-Yun Sun,
  Li~Fei-Fei, Nancy Kanwisher, Joshua~B. Tenenbaum, Daniel~LK Yamins, and
  Judith~E Fan.
\newblock Physion: Evaluating physical prediction from vision in humans and
  machines.
\newblock In {\em Thirty-fifth Conference on Neural Information Processing
  Systems Datasets and Benchmarks Track (Round 1)}, 2021.

\bibitem{benton2020learning}
Gregory Benton, Marc Finzi, Pavel Izmailov, and Andrew~G Wilson.
\newblock Learning invariances in neural networks from training data.
\newblock In H.~Larochelle, M.~Ranzato, R.~Hadsell, M.F. Balcan, and H.~Lin,
  editors, {\em Advances in Neural Information Processing Systems}, volume~33,
  pages 17605--17616. Curran Associates, Inc., 2020.

\bibitem{cesa2021program}
Gabriele Cesa, Leon Lang, and Maurice Weiler.
\newblock A program to build e (n)-equivariant steerable cnns.
\newblock In {\em International Conference on Learning Representations}, 2021.

\bibitem{chen2021grounding}
Zhenfang Chen, Jiayuan Mao, Jiajun Wu, Kwan-Yee~Kenneth Wong, Joshua~B.
  Tenenbaum, and Chuang Gan.
\newblock Grounding physical concepts of objects and events through dynamic
  visual reasoning.
\newblock In {\em International Conference on Learning Representations}, 2021.

\bibitem{chen2022comphy}
Zhenfang Chen, Kexin Yi, Yunzhu Li, Mingyu Ding, Antonio Torralba, Joshua~B.
  Tenenbaum, and Chuang Gan.
\newblock Comphy: Compositional physical reasoning of objects and events from
  videos.
\newblock In {\em International Conference on Learning Representations}, 2022.

\bibitem{cybenko1989approximation}
George Cybenko.
\newblock Approximation by superpositions of a sigmoidal function.
\newblock {\em Mathematics of control, signals and systems}, 2(4):303--314,
  1989.

\bibitem{de2020natural}
Pim de~Haan, Taco~S Cohen, and Max Welling.
\newblock Natural graph networks.
\newblock {\em Advances in Neural Information Processing Systems},
  33:3636--3646, 2020.

\bibitem{finzi2021practical}
Marc Finzi, Max Welling, and Andrew~Gordon Wilson.
\newblock A practical method for constructing equivariant multilayer
  perceptrons for arbitrary matrix groups.
\newblock In {\em International Conference on Machine Learning}, pages
  3318--3328. PMLR, 2021.

\bibitem{fischler1981random}
Martin~A Fischler and Robert~C Bolles.
\newblock Random sample consensus: a paradigm for model fitting with
  applications to image analysis and automated cartography.
\newblock {\em Communications of the ACM}, 24(6):381--395, 1981.

\bibitem{fuchs2020se3}
Fabian Fuchs, Daniel Worrall, Volker Fischer, and Max Welling.
\newblock Se(3)-transformers: 3d roto-translation equivariant attention
  networks.
\newblock In {\em Advances in Neural Information Processing Systems},
  volume~33, pages 1970--1981. Curran Associates, Inc., 2020.

\bibitem{gan2021threedworld}
Chuang Gan, Jeremy Schwartz, Seth Alter, Damian Mrowca, Martin Schrimpf, James
  Traer, Julian~De Freitas, Jonas Kubilius, Abhishek Bhandwaldar, Nick Haber,
  Megumi Sano, Kuno Kim, Elias Wang, Michael Lingelbach, Aidan Curtis,
  Kevin~Tyler Feigelis, Daniel Bear, Dan Gutfreund, David~Daniel Cox, Antonio
  Torralba, James~J. DiCarlo, Joshua~B. Tenenbaum, Josh Mcdermott, and
  Daniel~LK Yamins.
\newblock Three{DW}orld: A platform for interactive multi-modal physical
  simulation.
\newblock In {\em Thirty-fifth Conference on Neural Information Processing
  Systems Datasets and Benchmarks Track (Round 1)}, 2021.

\bibitem{gilmer2017neural}
Justin Gilmer, Samuel~S Schoenholz, Patrick~F Riley, Oriol Vinyals, and
  George~E Dahl.
\newblock Neural message passing for quantum chemistry.
\newblock In {\em International conference on machine learning}, pages
  1263--1272. PMLR, 2017.

\bibitem{hamilton2017inductive}
Will Hamilton, Zhitao Ying, and Jure Leskovec.
\newblock Inductive representation learning on large graphs.
\newblock {\em Advances in neural information processing systems}, 30, 2017.

\bibitem{han2022geometrically}
Jiaqi Han, Yu~Rong, Tingyang Xu, and Wenbing Huang.
\newblock Geometrically equivariant graph neural networks: A survey.
\newblock {\em arXiv preprint arXiv:2202.07230}, 2022.

\bibitem{hornik1991approximation}
Kurt Hornik.
\newblock Approximation capabilities of multilayer feedforward networks.
\newblock {\em Neural networks}, 4(2):251--257, 1991.

\bibitem{huang2022equivariant}
Wenbing Huang, Jiaqi Han, Yu~Rong, Tingyang Xu, Fuchun Sun, and Junzhou Huang.
\newblock Equivariant graph mechanics networks with constraints.
\newblock In {\em International Conference on Learning Representations}, 2022.

\bibitem{kipf2018neural}
Thomas Kipf, Ethan Fetaya, Kuan-Chieh Wang, Max Welling, and Richard Zemel.
\newblock Neural relational inference for interacting systems.
\newblock In {\em Proceedings of the 35th International Conference on Machine
  Learning}, volume~80 of {\em Proceedings of Machine Learning Research}, pages
  2688--2697. PMLR, 10--15 Jul 2018.

\bibitem{lerer2016learning}
Adam Lerer, Sam Gross, and Rob Fergus.
\newblock Learning physical intuition of block towers by example, 2016.

\bibitem{li2022contact}
Sizhe Li, Zhiao Huang, Tao Du, Hao Su, Joshua~B. Tenenbaum, and Chuang Gan.
\newblock Contact points discovery for soft-body manipulations with
  differentiable physics.
\newblock In {\em International Conference on Learning Representations}, 2022.

\bibitem{li2020visual}
Yunzhu Li, Toru Lin, Kexin Yi, Daniel Bear, Daniel Yamins, Jiajun Wu, Joshua
  Tenenbaum, and Antonio Torralba.
\newblock Visual grounding of learned physical models.
\newblock In {\em International conference on machine learning}, pages
  5927--5936. PMLR, 2020.

\bibitem{li2018learning}
Yunzhu Li, Jiajun Wu, Russ Tedrake, Joshua~B Tenenbaum, and Antonio Torralba.
\newblock Learning particle dynamics for manipulating rigid bodies, deformable
  objects, and fluids.
\newblock In {\em International Conference on Learning Representations}, 2018.

\bibitem{lin2022diffskill}
Xingyu Lin, Zhiao Huang, Yunzhu Li, David Held, Joshua~B. Tenenbaum, and Chuang
  Gan.
\newblock Diffskill: Skill abstraction from differentiable physics for
  deformable object manipulations with tools.
\newblock In {\em International Conference on Learning Representations}, 2022.

\bibitem{ma2022risp}
Pingchuan Ma, Tao Du, Joshua~B Tenenbaum, Wojciech Matusik, and Chuang Gan.
\newblock Risp: Rendering-invariant state predictor with differentiable
  simulation and rendering for cross-domain parameter estimation.
\newblock {\em ICLR}, 2022.

\bibitem{mitton2021local}
Joshua Mitton and Roderick Murray-Smith.
\newblock Local permutation equivariance for graph neural networks.
\newblock {\em arXiv preprint arXiv:2111.11840}, 2021.

\bibitem{mrowca2018flexible}
Damian Mrowca, Chengxu Zhuang, Elias Wang, Nick Haber, Li~F Fei-Fei, Josh
  Tenenbaum, and Daniel~L Yamins.
\newblock Flexible neural representation for physics prediction.
\newblock {\em Advances in neural information processing systems}, 31, 2018.

\bibitem{pfaff2021learning}
Tobias Pfaff, Meire Fortunato, Alvaro Sanchez-Gonzalez, and Peter Battaglia.
\newblock Learning mesh-based simulation with graph networks.
\newblock In {\em International Conference on Learning Representations}, 2021.

\bibitem{rubanova2022constraintbased}
Yulia Rubanova, Alvaro Sanchez-Gonzalez, Tobias Pfaff, and Peter Battaglia.
\newblock Constraint-based graph network simulator, 2022.

\bibitem{sanchez2019hamiltonian}
Alvaro Sanchez-Gonzalez, Victor Bapst, Kyle Cranmer, and Peter Battaglia.
\newblock Hamiltonian graph networks with ode integrators.
\newblock {\em arXiv preprint arXiv:1909.12790}, 2019.

\bibitem{sanchez2020learning}
Alvaro Sanchez-Gonzalez, Jonathan Godwin, Tobias Pfaff, Rex Ying, Jure
  Leskovec, and Peter Battaglia.
\newblock Learning to simulate complex physics with graph networks.
\newblock In {\em International Conference on Machine Learning}, pages
  8459--8468. PMLR, 2020.

\bibitem{sanchez2018graph}
Alvaro Sanchez-Gonzalez, Nicolas Heess, Jost~Tobias Springenberg, Josh Merel,
  Martin Riedmiller, Raia Hadsell, and Peter Battaglia.
\newblock Graph networks as learnable physics engines for inference and
  control.
\newblock In {\em Proceedings of the 35th International Conference on Machine
  Learning}, volume~80 of {\em Proceedings of Machine Learning Research}, pages
  4470--4479. PMLR, 10--15 Jul 2018.

\bibitem{satorras2021en}
V\'{\i}ctor~Garcia Satorras, Emiel Hoogeboom, and Max Welling.
\newblock E(n) equivariant graph neural networks.
\newblock In {\em Proceedings of the 38th International Conference on Machine
  Learning}, volume 139 of {\em Proceedings of Machine Learning Research},
  pages 9323--9332. PMLR, 18--24 Jul 2021.

\bibitem{thiede2021autobahn}
Erik~Henning Thiede, Wenda Zhou, and Risi Kondor.
\newblock Autobahn: Automorphism-based graph neural nets.
\newblock In A.~Beygelzimer, Y.~Dauphin, P.~Liang, and J.~Wortman Vaughan,
  editors, {\em Advances in Neural Information Processing Systems}, 2021.

\bibitem{thomas2018tensor}
Nathaniel Thomas, Tess Smidt, Steven Kearnes, Lusann Yang, Li~Li, Kai Kohlhoff,
  and Patrick Riley.
\newblock Tensor field networks: Rotation-and translation-equivariant neural
  networks for 3d point clouds.
\newblock {\em arXiv preprint arXiv:1802.08219}, 2018.

\bibitem{villar2021scalars}
Soledad Villar, David~W Hogg, Kate Storey-Fisher, Weichi Yao, and Ben
  Blum-Smith.
\newblock Scalars are universal: Equivariant machine learning, structured like
  classical physics.
\newblock In {\em Advances in Neural Information Processing Systems}, 2021.

\bibitem{weiler2019general}
Maurice Weiler and Gabriele Cesa.
\newblock General e (2)-equivariant steerable cnns.
\newblock {\em Advances in Neural Information Processing Systems}, 32, 2019.

\bibitem{ye2020object}
Yufei Ye, Dhiraj Gandhi, Abhinav Gupta, and Shubham Tulsiani.
\newblock Object-centric forward modeling for model predictive control.
\newblock In Leslie~Pack Kaelbling, Danica Kragic, and Komei Sugiura, editors,
  {\em Proceedings of the Conference on Robot Learning}, volume 100 of {\em
  Proceedings of Machine Learning Research}, pages 100--109. PMLR, 30 Oct--01
  Nov 2020.

\bibitem{yi2019clevrer}
Kexin Yi, Chuang Gan, Yunzhu Li, Pushmeet Kohli, Jiajun Wu, Antonio Torralba,
  and Joshua~B Tenenbaum.
\newblock {CLEVRER}: Collision events for video representation and reasoning.
\newblock {\em ICLR}, 2019.

\bibitem{zhong2021extending}
Yaofeng~Desmond Zhong, Biswadip Dey, and Amit Chakraborty.
\newblock Extending lagrangian and hamiltonian neural networks with
  differentiable contact models.
\newblock In A.~Beygelzimer, Y.~Dauphin, P.~Liang, and J.~Wortman Vaughan,
  editors, {\em Advances in Neural Information Processing Systems}, 2021.

\end{thebibliography}

\clearpage
\appendix 

\setcounter{equation}{12}
\setcounter{table}{3}
\setcounter{figure}{6}


\section{Theoretical Preliminaries and Proofs}
\label{sec:proof}

In the main paper, we have sketched the definitions and conclusions related to E(3) equivariance and subequivariance. Here, we introduce more details to facilitate the understanding for these conceptions. 

We first recap the orthogonal group $\text{O}(3)=\{\mO\in\R^{3\times 3}\mid \mO^\top\mO=\mI_3\}$ and the translation group $\text{T}(3)=\{\vt\in\R^3\}$. Then the Euclidean group is given by $\text{E}(3)=\text{O}(3)\ltimes\text{T}(3)$, where $\ltimes$ denotes the semidirect product. Basically, E(3) is a group of orthogonal transformations (rotations and reflections) and translations. 

Definition~\ref{def:equ} is rewritten as:
\setcounter{definition}{0}
\begin{definition}[Equivariance and Invariance]
We call that the function $f:\R^{3\times m}\times\R^{n}\rightarrow \R^{3 \times m'}$ is $G$-equivariant, if for any transformation $g$ in a group $\text{G}$, $f(g\cdot\Vec{\mZ},\vh)=g\cdot f(\Vec{\mZ},\vh)$, $\forall\Vec{\mZ}\in\R^{3\times m}$, $\forall\vh\in\R^{n}$. Similarly, $f$ is invariant if $f(g\cdot\Vec{\mZ},\vh)= f(\Vec{\mZ},\vh)$, $\forall g\in \text{G}$.  
\end{definition}

Eq.~(\ref{eq:GMN}) in the main paper is repeated as:
\begin{align}
    f(\Vec{\mZ},\vh)\coloneqq\Vec{\mZ}\mV, \quad\text{s.t.} \mV=\sigma(\Vec{\mZ}^{\top}\Vec{\mZ},\vh),
\end{align}
where the Multi-Layer Perceptron (MLP) function $\sigma:\sR^{m\times m+n}\mapsto\sR^{m\times m'}$ produces $\mV\in\R^{m\times m'}$. It is easy to justify that the function defined in Eq.~(\ref{eq:GMN}) is O(3)-equivariant. Moreover, by joining the analyses in GMN along with the conclusion by~\cite{villar2021scalars}, we immediately have the universality of the formulation in Eq.~(\ref{eq:GMN}):
\begin{proposition} [\cite{huang2022equivariant,villar2021scalars}]
\label{th:equ}
Let $f$ be defined by Eq.~(\ref{eq:GMN}). For any O(3)-equivariant function $\hat{f}(\Vec{\mZ},\vh)$, there always exists an MLP $\sigma$ satisfying $\|\hat{f}-f\|<\epsilon$ for arbitrarily small positive value $\epsilon$.
\end{proposition}

To prove this theorem, we require the following two lemmas:
\setcounter{theorem}{0}
\begin{lemma}
\label{le:span}
For any O(3)-equivariant function $\hat{f}(\Vec{\mZ},\vh)$, it must fall into the subspace spanned by the columns of $\Vec{\mZ}$, namely, there exists a function $s(\Vec{\mZ},\vh)$, satisfying $\hat{f}(\Vec{\mZ},\vh)=\Vec{\mZ}s(\Vec{\mZ},\vh)$.
\end{lemma}
\begin{proof}
The proof is given by~\cite{villar2021scalars}. Essentially, suppose $\Vec{\mZ}^{\perp}$ is the orthogonal complement of the column space of $\Vec{\mZ}$. Then there must exit functions $s(\Vec{\mZ},\vh)$ and $s^{\perp}(\Vec{\mZ},\vh)$, satisfying $\hat{f}(\Vec{\mZ},\vh)=\Vec{\mZ} s(\Vec{\mZ},\vh) + \Vec{\mZ}^{\perp}s^{\perp}(\Vec{\mZ},\vh)$. We can always find an orthogonal transformation $\mO$ allowing $\mO\Vec{\mZ}=\Vec{\mZ}$ while $\mO\Vec{\mZ}^{\perp}=-\Vec{\mZ}^{\perp}$. With this transformation $\mO$, we have $\hat{f}(\mO\Vec{\mZ},\vh)=\hat{f}(\Vec{\mZ},\vh)=\Vec{\mZ} s(\Vec{\mZ},\vh) + \Vec{\mZ}^{\perp}s^{\perp}(\Vec{\mZ},\vh)$, and $\mO\hat{f}(\Vec{\mZ},\vh)=\Vec{\mZ} s(\Vec{\mZ},\vh) - \Vec{\mZ}^{\perp}s^{\perp}(\Vec{\mZ},\vh)$. The equivariance property of $\hat{f}$ implies $\Vec{\mZ} s(\Vec{\mZ},\vh) + \Vec{\mZ}^{\perp}s^{\perp}(\Vec{\mZ},\vh)=\Vec{\mZ} s(\Vec{\mZ},\vh) - \Vec{\mZ}^{\perp}s^{\perp}(\Vec{\mZ},\vh)$, which derives $s^{\perp}(\Vec{\mZ},\vh)=0$. Hence, the proof is concluded. 
\end{proof}

\begin{lemma}
\label{le:inv}
If the O(3)-equivariant function $\hat{f}(\Vec{\mZ},\vh)$ lies in the subspace spanned by the columns of $\Vec{\mZ}$, then there exists a function $\sigma$ satisfying $\hat{f}(\Vec{\mZ},\vh) =\Vec{\mZ}\sigma(\Vec{\mZ}^\top\mZ, \vh)$.
\end{lemma}
\begin{proof}
The proof is provided by Corollary 2 in~\cite{huang2022equivariant}. The basic idea is that $\hat{f}(\Vec{\mZ},\vh)$ can be transformed to $\hat{f}(\Vec{\mZ},\vh)=\Vec{\mZ}\eta(\Vec{\mZ},\vh)$ where $\eta(\Vec{\mZ},\vh)$ is O(3)-invariant. According to Lemma 1-2 in~\cite{huang2022equivariant}, $\eta(\Vec{\mZ},\vh)$ must be written as $\eta(\Vec{\mZ},\vh)=\sigma(\Vec{\mZ}^\top\Vec{\mZ},\vh)$, which completes the proof. 
\end{proof}

With Lemma~\ref{le:span}-\ref{le:inv}, we always have $\hat{f}(\Vec{\mZ},\vh)=\Vec{\mZ} \sigma(\Vec{\mZ}^\top\Vec{\mZ},\vh)$. Since $\sigma$ can be universally approximated by MLP, then the conclusion in Proposition~\ref{th:equ} is proved. 

Nevertheless, the full symmetry is not always guaranteed and could be broken in certain directions owing to external force fields.
For example, the existence of gravity breaks the symmetry by exerting a force field along the gravitational axis $\vec{\vg}\in\R^3$, and the dynamics of the system will naturally preserve a gravitational acceleration in the vertical direction. By this means, the orthogonal symmetry is no longer maintained in every direction but only restricted to the subgroup $\text{O}_{\vec{\vg}}(3)\coloneqq\{\mO\in O(3) | \mO\vec{\vg}=\vec{\vg}\}$, that is, the rotations/reflections around the gravitational axis. We term such a reduction of equivariance as a novel notion: \emph{subequivariance}. 

\begin{definition}[Subequivariance and Subinvariance]
\label{def:subequ}
We call the function $f:\R^{3\times m}\times\R^n\rightarrow \R^{3 \times m'}$ is $O(3)$-subequivariant induced by $\vec{\vg}$, if $f(\mO\Vec{\mZ},\vh)=\mO f(\Vec{\mZ},\vh)$,  $\forall\mO\in \text{O}_{\vec{\vg}}(3)$, $\forall\Vec{\mZ}\in\R^{3\times m}$, $\forall\vh\in\R^n$; and similarly, it is $O(3)$-subinvariant induced by $\vec{\vg}$, if $f(\mO\Vec{\mZ}, \vh)=f(\Vec{\mZ}, \vh)$,  $\forall\mO\in \text{O}_{\vec{\vg}}(3)$.
\end{definition}

Eq.~(\ref{eq:GMN}) is clearly $\text{O}(3)$-subequivariant, but the $\text{O}(3)$-subequivariant function is unnecessarily the form like Eq.~(\ref{eq:GMN}). Considering the example for the gravity itself which maps all particles to the direction $\vec{\vg}$. It is natural to see that the function $f$ by Eq.~(\ref{eq:GMN}) fails to represent $\vec{\vg}$ if $\vec{\vg}$ does not fall into the subsubspace spanned by the columns of $\Vec{\mZ}$. While this example provides a failure, it also inspires us to derive the following augmented version upon Eq.~(\ref{eq:GMN}):
\begin{align}
    f_{\vec{\vg}}(\Vec{\mZ},\vh)=[\Vec{\mZ}, \vec{\vg}]\mV_{\vec{\vg}}, \quad\text{s.t.} \mV_{\vec{\vg}}=    \sigma([\Vec{\mZ},\vec{\vg}]^{\top}[\Vec{\mZ},\vec{\vg}],\vh),
\end{align}
where $\sigma:\R^{(m+1)\times(m+1)}\rightarrow\R^{(m+1)\times m'}$ is an MLP. Compared with Eq.~(\ref{eq:GMN}), here we just augment the directional input with $\vec{\vg}$. Interestingly, such a simple augmentation is universally expressive, which is proved in the following section. 

\subsection{Proof of Theorem~\ref{th:subequ}}
\label{sec:proof1}
\setcounter{theorem}{0}
\begin{theorem}
Let $f_{\vec{\vg}}(\Vec{\mZ},\vh)$ be defined by Eq.~(\ref{eq:subGMN}). Then, $f_{\vec{\vg}}$ is $\text{O}_{\vec{\vg}}(3)$-equivariant. More importantly, For any $\text{O}_{\vec{\vg}}(3)$-equivariant function $\hat{f}(\Vec{\mZ},\vh)$, there always exists an MLP $\sigma$ satisfying $\|\hat{f}-f_{\vec{\vg}}\|<\epsilon$ for arbitrarily small positive value $\epsilon$.
\end{theorem} 

The proof is similar to Proposition~\ref{th:equ} but with certain extensions. We first derive the following three lemmas.

\setcounter{theorem}{2}
\begin{lemma}
\label{le:s-span}
For any $\text{O}_{\vec{\vg}}(3)$-equivariant function $\hat{f}(\Vec{\mZ},\vh)$, it must fall into the subspace spanned by the columns of $[\Vec{\mZ},\vec{\vg}]$, namely, there exists a function $s(\Vec{\mZ},\vh)$, satisfying $\hat{f}(\Vec{\mZ},\vh)=[\Vec{\mZ},\vec{\vg}] s(\Vec{\mZ},\vh)$.
\end{lemma}
\begin{proof}
The proof is similar to Lemma~\ref{le:s-span}.
Suppose $\Vec{\mZ}^{\perp}$ is the orthogonal complement of the column space of $[\Vec{\mZ},\vec{\vg}]$. Then there must exit functions $s(\Vec{\mZ},\vh)$ and $s^{\perp}(\Vec{\mZ},\vh)$, satisfying $\hat{f}(\Vec{\mZ},\vh)=[\Vec{\mZ},\vec{\vg}] s(\Vec{\mZ},\vh) + \Vec{\mZ}^{\perp}s^{\perp}(\Vec{\mZ},\vh)$. We can always find an orthogonal transformation $\mO\in\text{O}_{\vec{\vg}}$ allowing $\mO\Vec{\mZ}=\Vec{\mZ}$, $\mO\vec{\vg}=\vec{\vg}$ while $\mO\Vec{\mZ}^{\perp}=-\Vec{\mZ}^{\perp}$. With this transformation $\mO$, we have $\hat{f}(\mO\Vec{\mZ},\vh)=\hat{f}(\Vec{\mZ},\vh)=[\Vec{\mZ},\vec{\vg}] s(\Vec{\mZ},\vh) + \Vec{\mZ}^{\perp}s^{\perp}(\Vec{\mZ},\vh)$, and $\mO\hat{f}(\Vec{\mZ},\vh)= [\Vec{\mZ},\vec{\vg}] s(\Vec{\mZ},\vh) - \Vec{\mZ}^{\perp}s^{\perp}(\Vec{\mZ},\vh)$. The equivariance property of $\hat{f}$ implies $[\Vec{\mZ},\vec{\vg}] s(\Vec{\mZ},\vh) + \Vec{\mZ}^{\perp}s^{\perp}(\Vec{\mZ},\vh)=[\Vec{\mZ},\vec{\vg}] s(\Vec{\mZ},\vh) - \Vec{\mZ}^{\perp}s^{\perp}(\Vec{\mZ},\vh)$, which derives $s^{\perp}(\Vec{\mZ},\vh)=0$. Hence, the proof is concluded. 
\end{proof}

\begin{lemma}
\label{le:s-inv}
If the $\text{O}_{\vec{\vg}}(3)$-equivariant function $\hat{f}(\Vec{\mZ},\vh)$ lies in the subspace spanned by the columns of $[\Vec{\mZ},\vec{\vg}]$, then there exists a $\text{O}_{\vec{\vg}}(3)$-invariant function $\eta$ satisfying $\hat{f}(\Vec{\mZ},\vh) =[\Vec{\mZ},\vec{\vg}]\eta(\Vec{\mZ}, \vh)$.
\end{lemma}
\begin{proof}
We assume the rank of $[\Vec{\mZ},\vec{\vg}]$ is $r$ ($r\leq\min\{3, m+1\}$). By performing the compact SVD decomposition on $[\Vec{\mZ},\vec{\vg}]$, we devise $[\Vec{\mZ},\vec{\vg}]=\mU\mS_r\mV^\top$, where $\mS_r\in\R^{r\times r}$ is a square diagonal matrix with positive diagonal elements, $\mU\in\R^{3\times r}, \mV\in\R^{(m+1)\times r}$, and $\mU^\top\mU=\mV^\top\mV=\mI_r$.

Since $\hat{f}(\Vec{\mZ},\vh)$ lies in the subspace spanned by the columns of $[\Vec{\mZ},\vec{\vg}]$, there exists a function $s(\Vec{\mZ},\vh)$ allowing $\hat{f}(\Vec{\mZ},\vh)=[\Vec{\mZ},\vec{\vg}]s(\Vec{\mZ},\vh)$. With applying the SVD decomposition, we have 
\begin{align}
    \label{eq:svd}
    \hat{f}(\Vec{\mZ},\vh)=\mU\mS_r\mV^\top s(\Vec{\mZ},\vh).
\end{align}
Given that $\hat{f}$ is  $\text{O}_{\vec{\vg}}(3)$-equivariant, it means $\hat{f}(\mO\Vec{\mZ},\vh)=\mO\hat{f}(\Vec{\mZ},\vh)$. By substituting this equation into Eq.~(\ref{eq:svd}),
\begin{align}
    & \label{eq:svd-equi}
    \mO\mU\mS_r\mV^\top s(\mO\Vec{\mZ},\vh) = \mO\mU\mS_r\mV^\top s(\Vec{\mZ},\vh) \notag \\
    \Leftrightarrow &  \mV^\top s(\mO\Vec{\mZ},\vh) = \mV^\top s(\Vec{\mZ},\vh).
\end{align}
As $\mV$ is $\text{O}_{\vec{\vg}}(3)$-invariant, $\mV^\top s(\Vec{\mZ},\vh)$ is an $\text{O}_{\vec{\vg}}(3)$-invariant function. Define $\eta(\Vec{\mZ},\vh)\coloneqq\mV\mV^\top s(\Vec{\mZ},\vh)$, which is apparently $\text{O}_{\vec{\vg}}(3)$-invariant. Then we have $\hat{f}(\Vec{\mZ},\vh)=\mU\mS_r\mV^\top s(\Vec{\mZ},\vh)=\mU\mS_r\mV^\top\mV \mV^\top s(\Vec{\mZ},\vh)=[\Vec{\mZ}, \vec{\vg}]\eta(\Vec{\mZ},\vh)$. 

\end{proof}

\begin{lemma}
\label{le:s-bject}
If the function $\eta(\Vec{\mZ},\vh)$ is $\text{O}_{\vec{\vg}}(3)$-invariant, it is of the form $\eta(\Vec{\mZ},\vh)=\sigma([\Vec{\mZ}, \vec{\vg}]^\top[\Vec{\mZ}, \vec{\vg}],\vh)$ for a certain function $\sigma$.
\end{lemma}
\begin{proof}
For any two inputs $\Vec{\mZ}_1, \Vec{\mZ}_2$, we claim that 
\begin{align}
\label{eq:one-one}
\exists\mO\in\text{O}_{\vec{\vg}}(3), \Vec{\mZ}_1=\mO\Vec{\mZ}_2 & \Leftrightarrow [\Vec{\mZ}_1, \vec{\vg}]^\top[\Vec{\mZ}_1, \vec{\vg}] = [\Vec{\mZ}_2, \vec{\vg}]^\top[\Vec{\mZ}_2, \vec{\vg}].   
\end{align}
The sufficiency direction $\Rightarrow$ is obvious. We only need to prove the necessity $\Leftarrow$.

We denote by $\Vec{\mG}^{\perp}\in\R^{3\times 2}$ the orthogonal complement of $\vec{\vg}$. Then, we have the decompositions:
\begin{align}
    \Vec{\mZ}_1 &= \vec{\vg}\bm{\alpha}_1 + \Vec{\mG}^{\perp}\bm{\beta}_1;\\
    \Vec{\mZ}_2 &= \vec{\vg}\bm{\alpha}_2 + \Vec{\mG}^{\perp}\bm{\beta}_2;
\end{align}
where $\bm{\alpha}_1,\bm{\alpha}_2\in\R^{1\times m}$, $\bm{\beta}_1,\bm{\beta}_2\in\R^{2\times m}$. The RHS of Eq.~(\ref{eq:one-one}) indicates $\vec{\vg}^\top\Vec{\mZ}_1=\vec{\vg}^\top\Vec{\mZ}_2$, and $\Vec{\mZ}_1^\top\Vec{\mZ}_1=\Vec{\mZ}_2^\top\Vec{\mZ}_2$, hence we have $\bm{\alpha}_1=\bm{\alpha}_2$, and 
\begin{align}
    \bm{\beta}_1^\top \bm{\beta}_1 
    & = (\Vec{\mG}^{\perp}\bm{\beta}_1)^\top (\Vec{\mG}^{\perp}\bm{\beta}_1)\notag\\
    & = (\Vec{\mZ}_1-\vec{\vg}\bm{\alpha}_1)^\top (\Vec{\mZ}_1-\vec{\vg}\bm{\alpha}_1) \notag\\
    &= \Vec{\mZ}_1^\top\Vec{\mZ}_1 - \bm{\alpha}_1^\top\bm{\alpha}_1 \notag \\
    &= \Vec{\mZ}_2^\top\Vec{\mZ}_2 - \bm{\alpha}_2^\top\bm{\alpha}_2 \notag \\
    & = \bm{\beta}_2^\top \bm{\beta}_2.
\end{align}
According to Lemma 1 in~\cite{huang2022equivariant}, $\bm{\beta}_1^\top \bm{\beta}_1=\bm{\beta}_2^\top \bm{\beta}_2\Leftrightarrow \exists\mO'\in\text{O}(2), \bm{\beta}_1=\mO'\bm{\beta}_2$.

Now we define $\mO=\vec{\vg}\vec{\vg}^\top+\Vec{\mG}^{\perp}\mO'(\Vec{\mG}^{\perp})^\top$. Clearly, $\mO^\top\mO=\vec{\vg}\vec{\vg}^\top+\Vec{\mG}^{\perp}(\Vec{\mG}^{\perp})^\top=\mI_3$, and $\mO^\top\vec{\vg}=\vec{\vg}$, which means $\mO\in\text{O}_{\vec{\vg}}(3)$. More interestingly,
\begin{align}
    \mO \Vec{\mZ}_2 &= \vec{\vg}\vec{\vg}^\top \Vec{\mZ}_2+\Vec{\mG}^{\perp}\mO'(\Vec{\mG}^{\perp})^\top \Vec{\mZ}_2 \notag \\
    &= \vec{\vg}\alpha_2 +  \Vec{\mG}^{\perp}\mO'\bm{\beta}_2 \notag \\
    &= \vec{\vg}\alpha_1 +  \Vec{\mG}^{\perp}\bm{\beta}_1 \notag \\
    &= \Vec{\mZ}_1, 
\end{align}
which concludes the proof of the claim. Fig.~\ref{fig:theo-fig} provides a geometrical explanation of our proof above, where the transformation $\mO'$ is actually the projection of $\mO$ onto the subspace $\Vec{\mG}^{\perp}$.

\begin{figure}[t]
    \centering
    \includegraphics[width=0.65\textwidth]{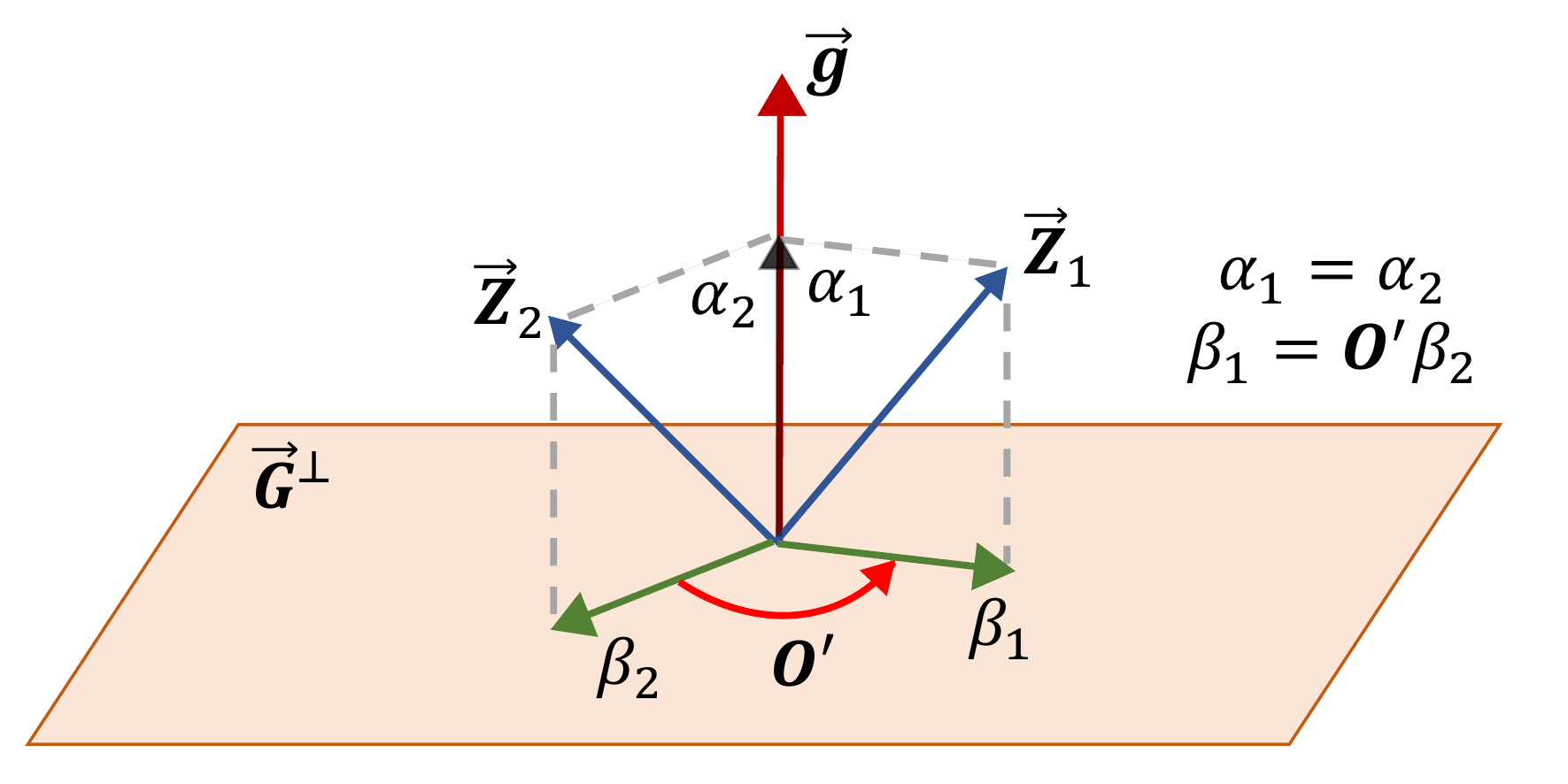}
    \caption{A geometric illustration of the proof for Lemma~\ref{le:s-bject}.}
    \label{fig:theo-fig}
\end{figure}

For any $\text{O}_{\vec{\vg}}(3)$-invariant function $\eta(\Vec{\mZ},\vh)$, it is a function of the equivalent class, that is, $\eta(\Vec{\mZ},\vh)=\sigma(\{\Vec{\mZ}\},\vh)$, where the equivalent class is defined as $\{\Vec{\mZ}\}\coloneqq \{\mO\Vec{\mZ}\mid \mO\in\text{O}_{\vec{\vg}}(3)\}$. The above claim in~(\ref{eq:one-one}) states that such equivalent class $\{\Vec{\mZ}\}$ maps to $[\Vec{\mZ}, \vec{\vg}]^\top[\Vec{\mZ}, \vec{\vg}]$ in a one-to-one way. Therefore, we must arrive at $\eta(\Vec{\mZ},\vh)=\sigma([\Vec{\mZ}, \vec{\vg}]^\top[\Vec{\mZ}, \vec{\vg}],\vh)$.

\end{proof}

By making use of Lemma~\ref{le:s-span}-\ref{le:s-bject}, we immediately obtain $f_{\vec{\vg}}(\Vec{\mZ},\vh)=[\Vec{\mZ},\vec{\vg}]\sigma([\Vec{\mZ},\vec{\vg}]^\top[\Vec{\mZ},\vec{\vg}],\vh)$. In accordance with the universality of MLP~\cite{cybenko1989approximation,hornik1991approximation}, we can always find an MLP that approximates $\sigma$ up to an accuracy $\frac{\epsilon}{G}$, where $G$ is the bound of $\|[\Vec{\mZ},\vec{\vg}]\|$, which implies  $\|\hat{f}-f_{\vec{\vg}}\|<\|[\Vec{\mZ},\vec{\vg}]\|\frac{\epsilon}{G}<\epsilon$, where $f_{\vec{\vg}}$ is defined in Eq.~(\ref{eq:subGMN}). The proof of Theorem~\ref{th:subequ} is finished.

Note that $f$ by Eq.~(\ref{eq:subGMN}) can also be considered as a function of both $\Vec{\mZ}$ and $\vec{\vg}$, and it is universal according to Proposition~\ref{th:equ}. When $f$ reduces to a function of $\Vec{\mZ}$ by fixing $\vec{\vg}$, then by Theorem~\ref{th:subequ}, it is still universal with respect to the subgroup that leaves $\vec{\vg}$ unchanged. This phenomenon shows the universality holds universally no matter which input variable of Eq.~(\ref{eq:subGMN}) we fix. 

Although this paper mainly focuses on the 3-dimension group O(3), our theorems and proofs above are generalizable to the $d$-dimension group O($d$), when $d$ is not restricted to 3 and the external field that breaks the symmetry has more than 1 direction, namely the gravity vector $\vec{\vg}\in\R^{3\times 1}$ becomes a directional matrix $\Vec{\mG}\in\R^{d \times d'} (d'<d)$.   


\subsection{Proof of Theorem~\ref{the:message-passing-equ}}
\label{sec:proof2}
\setcounter{theorem}{1}
\begin{theorem}
\label{the:message-passing-equ}
The message passing $\varphi$ (Eq.~\ref{eq:somp}) is $O_{\vec{\vg}}(3)$-equivariant.
\end{theorem}

\begin{proof}
We prove step by step in the specifications of Eq.~(\ref{eq:somp}) from Eq.~(\ref{eq:vector-interaction})-(\ref{eq:update-mlp}). For better clarity, we denote the variables after applying the transformation $\mO\in \text{O}_{\vec{\vg}}(3)$ with the superscript $\ast$. 

For any $\mO\in \text{O}_{\vec{\vg}}(3)$, we immediately have $(\vec{\mZ}^\ast_i, \vh^\ast_i) = (\mO\vec{\mZ}_i, \vh_i)$, and
\begin{align}
\nonumber
    (\vec{\mC}^\ast_k, \vc^\ast_k) &=\left(\frac{1}{|\{i:o(i)=k\}|}\sum\nolimits_{\{i:o(i)=k\}}\vec{\mZ}^\ast_i, \sum\nolimits_{\{i:o(i)=k\}}\vh^\ast_i\right),\\
    \nonumber
  &=\left(\frac{1}{|\{i:o(i)=k\}|}\sum\nolimits_{\{i:o(i)=k\}}\mO\vec{\mZ}_i, \sum\nolimits_{\{i:o(i)=k\}}\vh_i\right), \\
  \nonumber
    &= (\mO\vec{\mC}_k, \vc_k).
\end{align}

Therefore, for Eq.~(\ref{eq:vector-interaction}), 
\begin{align}
\nonumber
    \vec{\mZ}^\ast_{ij} &= (\vec{\mZ}^\ast_i\ominus\vec{\mC}^\ast_{o(i)}) \|(\vec{\mZ}^\ast_j\ominus\vec{\mC}^\ast_{o(j)})\| (\vec{\mZ}^\ast_i\ominus\vec{\mZ}^\ast_j),\\
    \nonumber
    &= (\mO\vec{\mZ}_i\ominus\mO\vec{\mC}_{o(i)}) \|(\mO\vec{\mZ}_j\ominus\mO\vec{\mC}_{o(j)})\| (\mO\vec{\mZ}_i\ominus\mO\vec{\mZ}_j), \\
    \nonumber
    &= \mO\left( (\vec{\mZ}_i\ominus\vec{\mC}_{o(i)}) \|(\vec{\mZ}_j\ominus\vec{\mC}_{o(j)})\| (\vec{\mZ}_i\ominus\vec{\mZ}_j)  \right) = \mO\vec{\mZ}_{ij},
\end{align}
which is $\text{O}_{\vec{\vg}}(3)$-equivariant.
Similarly for $\vh_{ij}$ in Eq.~(\ref{eq:scalar-interaction}), it is invariant, \emph{i.e.}, $\vh^\ast_{ij} = \vh_{ij}$. Since $\phi_{\vec{\vg}}$ in Eq.~(\ref{eq:message-mlp}) is designed to be subequivariant ($\text{O}_{\vec{\vg}}(3)$-equivariant), by definition we immediately derive that $(\vec{\mM}^\ast_{ij}, \vm^\ast_{ij}) = \phi_{\vec{\vg}}\left(\vec{\mZ}^\ast_{ij}, \vh^\ast_{ij} \right) = \phi_{\vec{\vg}}\left(\mO\vec{\mZ}_{ij}, \vh_{ij} \right) = (\mO\vec{\mM}_{ij}, \vm_{ij})$.

Finally, for the aggregation and update in Eq.~(\ref{eq:update-mlp}), it is derived as

\begin{align}
\nonumber
    (\vec{\mZ}'^\ast_i, \vh'^\ast_i) &= (\vec{\mZ}^\ast_i, \vh^\ast_i) + \psi_{\vec{\vg}}\left((\sum\nolimits_{j\in\gN(i)}\vec{\mM}^\ast_{ij}) \| (\vec{\mZ}^\ast_i\ominus\vec{\mC}^\ast_{o(i)}), (\sum\nolimits_{j\in\gN(i)}\vm^\ast_{ij}) \| \vh^\ast_i \| \vc^\ast_{o(i)} \right), \\
    \nonumber
    &= (\mO\vec{\mZ}_i,\vh_i) + \psi_{\vec{\vg}}\left((\sum\nolimits_{j\in\gN(i)}\mO\vec{\mM}_{ij}) \| (\mO\vec{\mZ}_i\ominus\mO\vec{\mC}_{o(i)}), (\sum\nolimits_{j\in\gN(i)}\vm_{ij}) \| \vh_i \| \vc_{o(i)} \right), \\
    \nonumber
    &= (\mO\vec{\mZ}_i,\vh_i) +
    \psi_{\vec{\vg}}\left(\mO((\sum\nolimits_{j\in\gN(i)}\vec{\mM}_{ij}) \| (\vec{\mZ}_i\ominus\vec{\mC}_{o(i)})), (\sum\nolimits_{j\in\gN(i)}\vm_{ij}) \| \vh_i \| \vc_{o(i)} \right), \\
    \nonumber
    &= (\mO\vec{\mZ}'_i, \vh'_i),
\end{align}
which concludes the proof by showing that $\vec{\mZ}'_i$ is $\text{O}_{\vec{\vg}}(3)$-equivariant and $\vh'_i$ is $\text{O}_{\vec{\vg}}(3)$-invariant.
\end{proof}

Indeed, by leveraging Theorem~\ref{the:message-passing-equ}, it is also straightforward that the resulting SGNN is also $\text{O}_{\vec{\vg}}(3)$-equivariant, since its components $\varphi_1$, $\varphi_2$, and $\varphi_3$ are all $\text{O}_{\vec{\vg}}(3)$-equivariant functions. 

\textbf{Remark on translation equivariance.} Regarding translation equivariance, the operation ``$\ominus$'' always results in translation-invariant representations. Therefore, $\vec{\mZ}_{ij}$ is translation invariant, and so does $\vh_{ij}$. Following this induction, the intermediate results until the output of $\psi_{\vec{\vg}}$ are all translation-invariant. By adding $(\vec{\mZ}_i, \vh_i)$ to the final output, it is clear to see that $\vec{\mZ}'_i$ is translation-equivariant and $\vh'_i$ is translation-invariant.

\subsection{Theoretical Comparisons Between EGNN, GMN, and SGNN}
\label{sec:proof3}
In this sub-section, we theoretically reveal that both EGNN and GMN are special cases of SGNN by choosing specific forms of MLP in $\phi_{\vec{\vg}}$ of Eq.~(\ref{eq:message-mlp}) and $\psi_{\vec{\vg}}$ of Eq.~(\ref{eq:update-mlp}). We provide an illustration from the architectural view in Fig.~\ref{fig:theo-compare}.

\begin{figure}[htbp]
    \centering
    \includegraphics[width=0.98\textwidth]{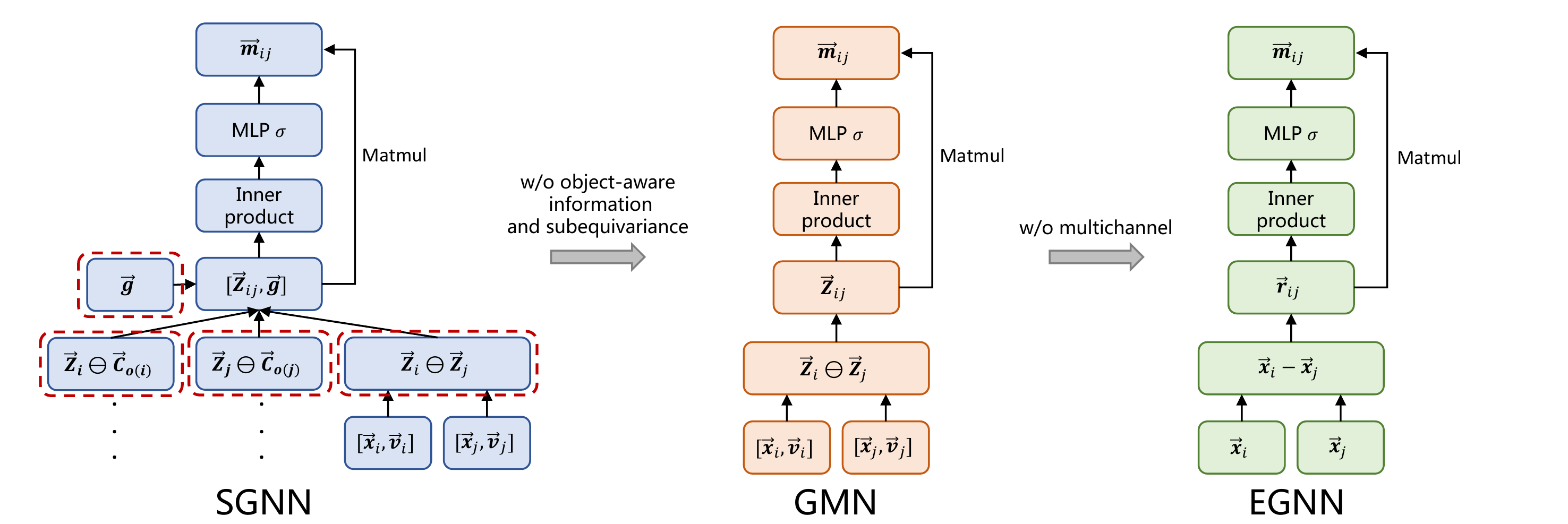}
    \caption{An illustration of the comparison between SGNN, GMN, and EGNN.}
    \label{fig:theo-compare}
\end{figure}

We first prove that SGNN can reduce to GMN by choosing specific form of the MLP $\sigma$. Intuitively, this reduction can be realized by masking the channels related to $\vec{\vg}$ and the object-aware information $\vec{\mC}_{o(i)}$ and $\vec{\mC}_{o(j)}$ from the input and the corresponding output before taking the matrix multiplication. We summarize this into the following lemma.

\setcounter{theorem}{5}
\begin{lemma}
\label{lemma:1}
Consider arbitrary 3D multichannel vectors $\vec{\mZ}_1\in\sR^{3\times m_1}$ and $\vec{\mZ}_2\in\sR^{3\times m_2}$. Let $f_{\sigma_1}=[\vec{\mZ}_1, \vec{\mZ}_2] \sigma_1 \left([\vec{\mZ}_1, \vec{\mZ}_2]^\top[\vec{\mZ}_1, \vec{\mZ}_2] \right), g_{\sigma_2} = \vec{\mZ}_1 \sigma_2 \left(\vec{\mZ}_1^\top \vec{\mZ}_1  \right)$. Then, for any $\sigma_2$, there exists $\sigma_1$, \emph{s.t.} $\|f_{\sigma_1} - g_{\sigma_2} \|<\epsilon$ for arbitrary small positive value $\epsilon$.
\end{lemma}

\begin{proof}
The input of $\sigma_1$ can be rewritten as $\begin{bmatrix} \vec{\mZ}_1^\top\vec{\mZ}_1&  \vec{\mZ}_1^\top\vec{\mZ}_2\\
    \vec{\mZ}_2^\top\vec{\mZ}_1 &  \vec{\mZ}_2^\top\vec{\mZ}_2 \end{bmatrix}$. Denote $f_{\text{in}}(\mX) = \begin{bmatrix}1 & 0 \\ 0 & 0 \end{bmatrix}\mX$ and $f_{\text{out}}(\mX) = \begin{bmatrix} \mX \\ 0 \end{bmatrix}$.
Then, when choosing $\sigma_1 = f_{\text{out}} \circ \sigma_2 \circ f_{\text{in}}$, we have

\begin{align}
\nonumber
f_{\sigma_1} &= [\vec{\mZ}_1, \vec{\mZ}_2] \sigma_1 \left([\vec{\mZ}_1, \vec{\mZ}_2]^\top[\vec{\mZ}_1, \vec{\mZ}_2] \right), \\
\nonumber
&= [\vec{\mZ}_1, \vec{\mZ}_2] f_{\text{out}} \circ \sigma_2 \circ f_{\text{in}} \left(\begin{bmatrix} \vec{\mZ}_1^\top\vec{\mZ}_1&  \vec{\mZ}_1^\top\vec{\mZ}_2\\
    \vec{\mZ}_2^\top\vec{\mZ}_1 &  \vec{\mZ}_2^\top\vec{\mZ}_2 \end{bmatrix} \right), \\
    \nonumber
&= [\vec{\mZ}_1, \vec{\mZ}_2] f_{\text{out}} \circ \sigma_2 \left(\vec{\mZ}_1^\top\vec{\mZ}_1  \right),\\
\nonumber
&= [\vec{\mZ}_1, \vec{\mZ}_2] \begin{bmatrix} \sigma_2 \left(\vec{\mZ}_1^\top\vec{\mZ}_1  \right)\\ 0 \end{bmatrix}, \\
\nonumber
&= \vec{\mZ}_1 \sigma_2 \left(\vec{\mZ}_1^\top\vec{\mZ}_1  \right) = g_{\sigma_2}.
\end{align}
    
Therefore, by the universal approximation of MLP~\cite{cybenko1989approximation,hornik1991approximation}, we know that for any $g_{\sigma_2}$ parameterized by $\sigma_2$, there exists $\sigma_1 = f_{\text{out}} \circ \sigma_2 \circ f_{\text{in}}$ that can approximate $g_{\sigma_2}$ with $f_{\sigma_1}$ by arbitrarily small error $\epsilon$.
\end{proof}

Leveraging this lemma directly gives the following theorem.

\setcounter{theorem}{2}
\begin{theorem}
Let the symbol $f_{\sigma_1} \succeq g_{\sigma_2}$ denotes that for any $g_{\sigma_2}$ parameterized by $\sigma_2$, there exists $\sigma_1$ satisfying $\|f_{\sigma_1} - g_{\sigma_2}\| < \epsilon$ for arbitrarily small positive value $\epsilon$. Then, SOMP $\succeq$ GMN $\succeq$ EGNN.
\end{theorem}

\begin{proof}
\textbf{1. SOMP $\succeq$ GMN.} Let $\vec{\mZ}_1 = [\vec{\mZ}_i \ominus \vec{\mZ}_j]$, and $\vec{\mZ}_2 = [\vec{\mZ}_i\ominus \vec{\mC}_{o(i)}, \vec{\mZ}_j\ominus \vec{\mC}_{o(j)}, \vec{\vg}]$. Using Lemma~\ref{lemma:1} immediately shows that SOMP $\succeq$ GMN.
\textbf{2. GMN $\succeq$ EGNN.} Similarly, we choose $\vec{\mZ}_1 = [\vec{\vx}_i - \vec{\vx}_j]$, and $\vec{\mZ}_2 = [\vec{\vv}_i, \vec{\vv}_j]$, and by Lemma~\ref{lemma:1} we have GMN $\succeq$ EGNN, which concludes the proof.
\end{proof}

By these theoretical derivations we are able to show that SOMP indeed has stronger expressivity than GMN and EGNN, by leveraging object-aware information as well as the subequivariance depicted by vector $\vec{\vg}$.

\section{Implementation Details}
\label{sec:impl-detail}

\subsection{Hyper-parameters and Training Details}

We utilize the codebase provided by Physion~\cite{bear2021physion} for particle-based methods\footnote{\url{https://github.com/htung0101/Physion-particles}}. This repository contains the implementation of GNS~\cite{sanchez2020learning} and DPI~\cite{li2018learning}. For DPI, we notice that it has been optimized in the RigidFall task by ~\cite{li2020visual}\footnote{\url{https://github.com/YunzhuLi/VGPL-Dynamics-Prior}}, and we thus adopt their optimized version on RigidFall. As for EGNN~\cite{satorras2021en} and GMN~\cite{huang2022equivariant}, we resort to their original implementations\footnote{\url{https://github.com/vgsatorras/egnn}}\footnote{\url{https://github.com/hanjq17/GMN}}, respectively. The datasets and code repositories are released under MIT license.

We basically follow the hyper-parameters suggested by Physion. In detail, for GNS and DPI, we use a hidden dimension of 200 for the node update function $\psi$ and 300 for the message computation function $\phi$, each of which consists of 3 layers with ReLU as the activation function. The iteration step is set to 10 for GNS and 2 for DPI due to its multi-stage hierarchical modeling. We use an Adam optimizer with initial learning rate 0.0001, betas (0.9, 0.999), and a Plateau scheduler with a patience of 3 epochs and decaying factor 0.8. For EGNN, GMN, and SGNN, we still build upon the above hyper-parameters with very minor modifications. We adopt hidden dimension 200 uniformly for $\psi$ and $\phi$ with SiLU activation function, and 4 iterations in $\varphi_1$, $\varphi_2$, and $\varphi_3$ for SGNN, while 10 for EGNN and GMN. We use an early-stopping of 10 epochs.
We use a batch size of 1 on Physion due to the large size of each system and 8 on RigidFall due to its relatively small size. Besides, we also inject noise during training for better test-time long-term rollout prediction, exactly following the settings of Physion and RigidFall~\cite{li2020visual}. The scale is set to 3e-4 in Physion and 0.05 of the std in RigidFall. The cutoff radius $\gamma$ is set to be 0.08 on both datasets. On both datasets, we only use the state information of last frame $t$ as input to predict the information of frame $t+1$. The experiments are conducted on single card NVIDIA Tesla V100 GPU.

Notably, for EGNN-S and GMN-S, we make a modification to their updates of the velocity, namely,
$$ \vec{\mathbf{v}}_i^{l+1}=\phi_v(\mathbf{h}^l_i)\vec{\mathbf{v}}_i^{l} + \underline{\phi_g(\mathbf{h}_i^l)\vec{\mathbf{g}}} + \sum_{j\in\mathcal{N}(i)}(\vec{\mathbf{x}}_i - \vec{\mathbf{x}}_j)\phi_x(\mathbf{m}_{ij}), $$
where the underlined term highlights our adaptation, $\phi_v, \phi_g, \phi_x$ are all MLPs, and the superscript $l$ indicates the iteration step. The intuition of this term is similar to adding a gravitational acceleration term to the update of velocity. This formulation meets $O_{\vec{\vg}}(3)$-equivariance.

As for the data splits, we strictly follow Physion  and RigidFall. In detail, for Physion, the full training set contains 2000 trajectories in each scenario, which is then split into training and validation with the ratio 9:1. The testing set contains 200 trajectories. For RigidFall,  the full training set contains 5000 trajectories, which is also split into training and validation with ratio 9:1. Particularly, to study the data-efficiency of different models, we sub-sample multiple training sets with sizes 200, 500, 1000, 5000, as illustrated in Table~\ref{tab:rigidfall_mse} in the paper.

Besides, in the implementation we also employ a normalization before feeding the inner product into the MLP, \emph{i.e.}, we normalize the inner product by $\vec{\mZ}^\top\vec{\mZ}/\| \vec{\mZ}^\top\vec{\mZ} \|_F$ in Eq.~(\ref{eq:GMN}), as suggested in GMN~\cite{huang2022equivariant} to control the expanding variance in scale brought by the inner product for better numerical stability. This is similarly adopted in Eq.~(\ref{eq:subGMN}) for $[\vec{\mZ}, \vec{\vg}]$, where we also propose to dynamically control the scale of $\vec{\vg}$ by $\eta(\vh)\in\sR$ where $\eta$ is a lightweight two-layer MLP. We find these considerations generally leads to faster convergence.

\subsection{Computational Complexity}
In this sub-section, we compare the computational budget of SGNN to those of the baselines, aiming to illustrate that the superior performance brought by SGNN stems from our design of subequivariance and hierarchical modeling, but not more computational overhead or more parameters used.

Generally, the computational complexity of the models is approached by $O(KT|\gE|)$, where $K$ is the number of stages, $T$ is the number of message passing steps in each stage, and $|\gE|$ measures the number of edges in the interaction graph.
Among all the models, they can be characterized into non-hierarchical methods including GNS, EGNN, and GMN, as well as hierarchical methods including DPI and our SGNN. For the non-hierarchical models, it has been observed in~\cite{sanchez2020learning} that generally a larger number of propagation iterations $T$ is required for better performance, which is set to 10 in the implementation. For the hierarchical methods, it does not require such a large number, which is set to 2 for DPI following their original setup and 4 in SGNN. By this means, we have carefully controlled the computational budget by making the total number of message passing iterations nearly the same for all methods, since DPI requires in total $K=4$ stages (leaf-leaf, leaf-root, root-root, root-leaf) with  each stage involving 2 steps, and SGNN requires $K=3$ stages $\varphi_1$, $\varphi_2$, and $\varphi_3$, with each stage having 4 steps. GNS, EGNN, and GMN only have one stage, but need 10 steps in this stage. Besides, it is also worth noticing that SGNN employs edge separation, which further brings down the cost when computing message passing with $|\gE_{\text{inter}}|, |\gE_{\text{inner}}|, |\gE_{\text{obj}}| < |\gE|$ edges.
Regarding the size of the networks, we reuse the hyper-parameters, \emph{e.g.}, the number of layers in MLP and the hidden dimension, of the baselines for SGNN, which makes SGNN nearly as the same size as DPI. 

To further illustrate, we provide the total number of parameters and the average training time per step in seconds on Physion Dominoes in Table~\ref{tab:param}, which indicates that SGNN has a moderate number of parameters and still enjoys fast training speed.
\begin{table}[htbp]
  \centering
  \caption{Number of parameters (\#Param) and average training time per step on Physion Dominoes.}
    \begin{tabular}{lccccc}
    \toprule
          & GNS~\cite{sanchez2020learning}   & DPI~\cite{li2018learning}   & EGNN~\cite{satorras2021en}  & GMN~\cite{huang2022equivariant}   & SGNN \\
    \midrule
    \#Param & 0.54M & 1.98M & 0.45M & 0.51M & 1.50M \\
    Time (seconds)  & 0.40 $\pm$ 0.02   & 0.35$\pm$0.03  & 0.08$\pm$0.01 & 0.09$\pm$0.01 & 0.11$\pm$ 0.02 \\
    \bottomrule
    \end{tabular}%
  \label{tab:param}%
\end{table}%

\section{More Experiment Results}

\subsection{Motivating Example}

\begin{figure}[htbp]
    \centering
    \includegraphics[width=0.98\textwidth]{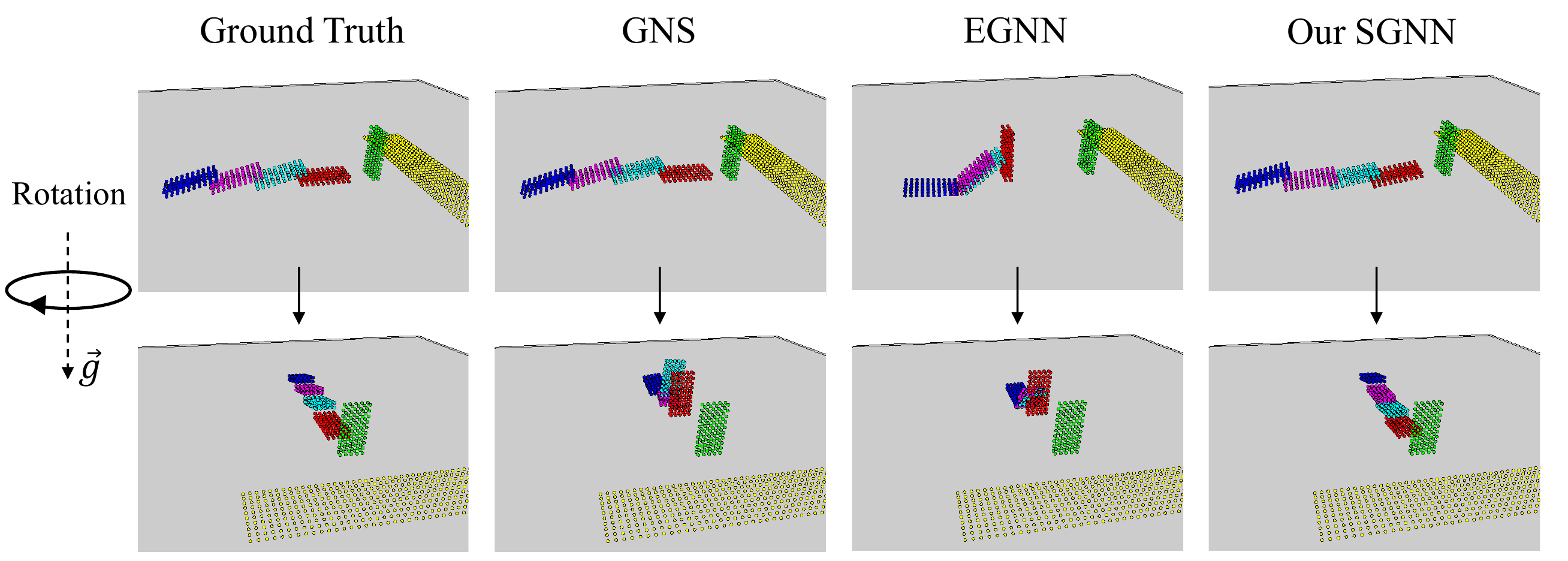}
    \caption{A motivating example on Physion Dominoes. GNS is able to produce accurate prediction only in the same direction as the training trajectories. EGNN fails to learn the complex interactions with gravity involved. Our model predicts very accurately regardless of any valid rotations applied.}
    \label{fig:motivation}
\end{figure}

We provide a motivating example in Fig.~\ref{fig:motivation}. We have the following observations: \textbf{1.} Physical laws abide by symmetry. The dynamics of the dominoes falling from left to right will exactly be preserved the same way if we apply a rotation along the gravitational axis $\vec{\vg}$; \textbf{2.} Without any guarantee of equivariance, models like GNS fail to preserve such symmetry. For example, if all dominoes in the training data are falling from left to right, GNS might be able to produce accurate prediction if the testing trajectory is also aligned from left to right, but will perform poorly if a rotation is adopted; \textbf{3.} Equivariant models like EGNN and GMN are able to preserve the symmetry, \emph{i.e.}, their prediction will rotate/translate together with the input. However, they are enforcing E($3$)-equivariance, which is too strong that limits the expressivity of the model when the equivariance is violated by the existence of gravity; \textbf{4.} Our model SGNN takes into consideration the desirable subequivariance as well as object-aware message passing, yielding very accurate prediction while being invulnerable to any test-time rotation along the vertical axis.

\subsection{Rollout MSE on Physion}
\label{sec:mse}

\begin{table}[htbp]
  \centering
  \caption{Rollout MSE on Physion when $t=10$.}
    \begin{tabular}{lcccccccc}
    \toprule
          & Dominoes & Contain & Link  & Drape & Support & Drop  & Collide & Roll \\
    \midrule
    GNS*~\cite{sanchez2020learning}  & 0.19  & 3.76  & 9.89  & 18.51  & 4.12  & 2.59  & 1.58  & 0.60  \\
    DPI*~\cite{li2018learning}  & 0.15  & 2.41  & 6.56  & 18.43  & \textbf{3.67}  & 2.39  & 1.66  & 0.63  \\
    \midrule
    GNS~\cite{sanchez2020learning}   & 1.01  & 4.25  & 16.23  & 23.72  & 4.20  & 2.57  & 5.35  & 0.80  \\
    DPI~\cite{li2018learning}   & 1.04  & 3.13  & 12.88  & 35.05  & 4.37  & 1.9   & 8.02  & 2.05  \\
    \midrule
    GNS-Rot~\cite{sanchez2020learning} & 0.27  & 3.95  & 10.39  & 29.85  & 4.69  & 2.04  & 1.53  & 0.63  \\
    DPI-Rot~\cite{li2018learning} & 0.22  & \textbf{2.27}  & 6.37  & 26.94  & 4.61  & 1.88  & 1.64  & 1.41  \\
    \midrule
    EGNN~\cite{satorras2021en}  & 0.31  & 5.91  & 12.93  & 39.81  & 5.25  & 1.86  & 4.22  & 1.01  \\
    GMN~\cite{huang2022equivariant}   & 0.59  & 8.88  & 19.36  & 39.70  & 9.08  & 3.16  & 6.13  & 2.03  \\
    \midrule
    SGNN  & \textbf{0.09}  & 2.32  & \textbf{4.98}  & \textbf{17.23}  & 4.52  & \textbf{1.37}  & \textbf{1.34}  & \textbf{0.53}  \\
    \bottomrule
    \end{tabular}%
  \label{tab:rollout10}%
\end{table}%

\begin{table}[htbp]
  \centering
  \caption{Rollout MSE ($\times 10^{1}$) on Physion when $t=35$.}
    \begin{tabular}{lcccccccc}
    \toprule
          & Dominoes & Contain & Link  & Drape & Support & Drop  & Collide & Roll \\
    \midrule
    GNS*~\cite{sanchez2020learning}  & 0.16 & 4.83  & 10.95  & 9.04  & \textbf{9.59}  & 4.78  & 3.28  & 0.39  \\
    DPI*~\cite{li2018learning}  & 0.14  & 2.36  & 9.49  & 32.97  & 28.97  & 1.73 & 4.07  & 0.38  \\
    \midrule
    GNS~\cite{sanchez2020learning}   & 0.53  & 5.28  & 16.56  & 9.53  & 9.74  & 5.92  & 8.10  & 0.43  \\
    DPI~\cite{li2018learning}   & 0.56  & 3.39  & 16.29  & 30.21  & 17.07  & 0.98  & 11.40  & 1.54  \\
    \midrule
    GNS-Rot~\cite{sanchez2020learning} & 0.22  & 4.32  & 15.54  & 12.13  & 9.80  & 1.79  & 3.02  & 0.37  \\
    DPI-Rot~\cite{li2018learning} & 0.35  & 4.65  & 11.04  & 12.02  & 53.69  & 1.24  & 3.68  & 1.54  \\
    \midrule
    EGNN~\cite{satorras2021en}  & 0.30  & 6.36  & 19.99  & 13.41  & 14.81  & 0.96  & 5.96  & 0.66  \\
    GMN~\cite{huang2022equivariant}   & 0.39  & 7.25  & 22.97  & 12.33  & 16.60  & 1.53  & 7.93  & 0.84  \\
    \midrule
    SGNN  & \textbf{0.07}  & \textbf{2.16}  & \textbf{7.01}  & \textbf{8.14}  & 13.55  & \textbf{0.70}  & \textbf{2.80}  & \textbf{0.31}  \\
    \bottomrule
    \end{tabular}%
  \label{tab:rollout35}%
\end{table}%

We provide the rollout MSE when $t=10$ and $t=35$ on Physion in Table~\ref{tab:rollout10} and~\ref{tab:rollout35}, respectively. Our SGNN gives the best results in 7 out of all 8 scenarios, especially favorable on long-term prediction.

\subsection{Learning curves}

\begin{figure}[htbp]
    \centering
    \includegraphics[width=0.42\textwidth]{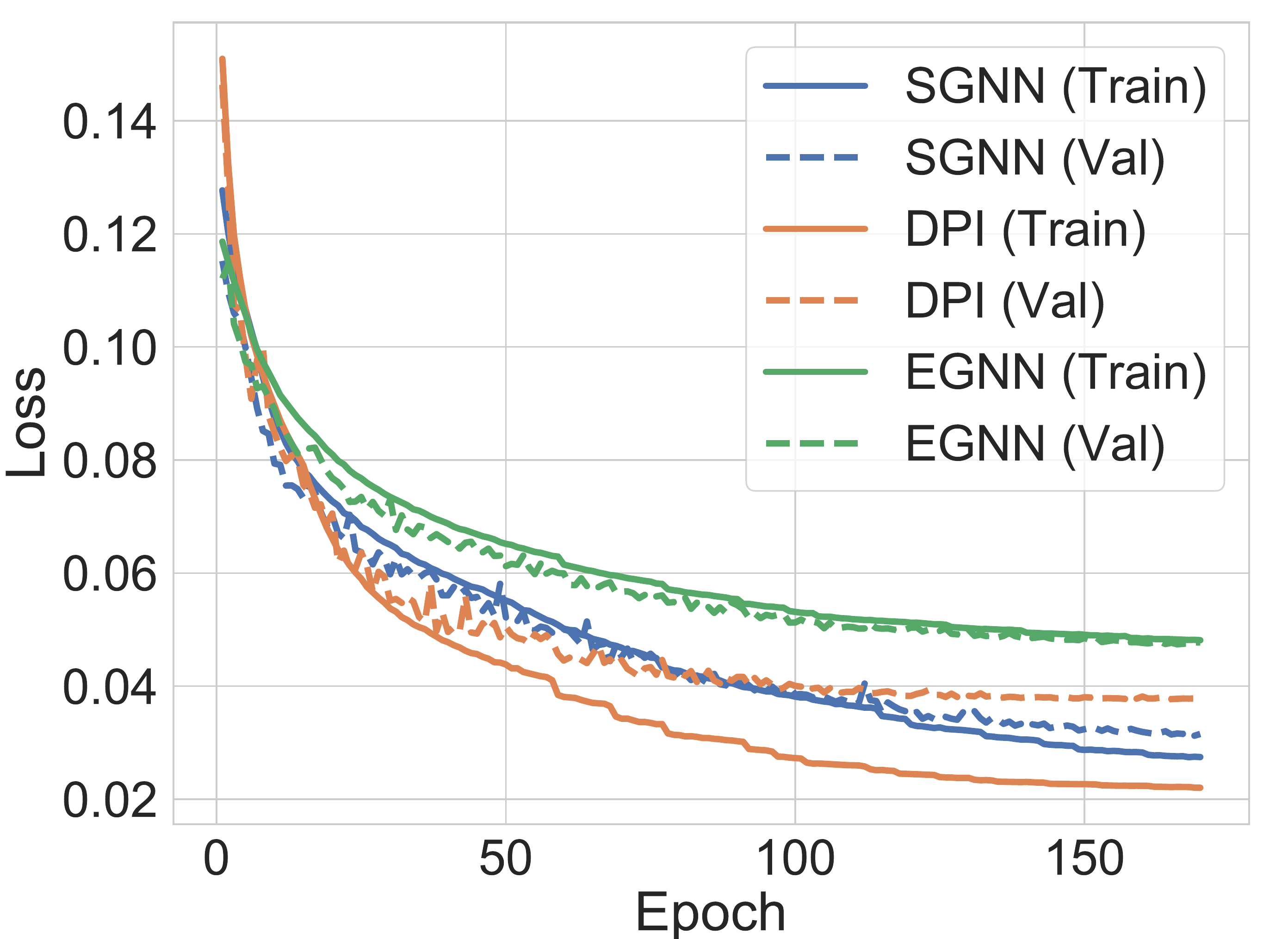}
    \includegraphics[width=0.42\textwidth]{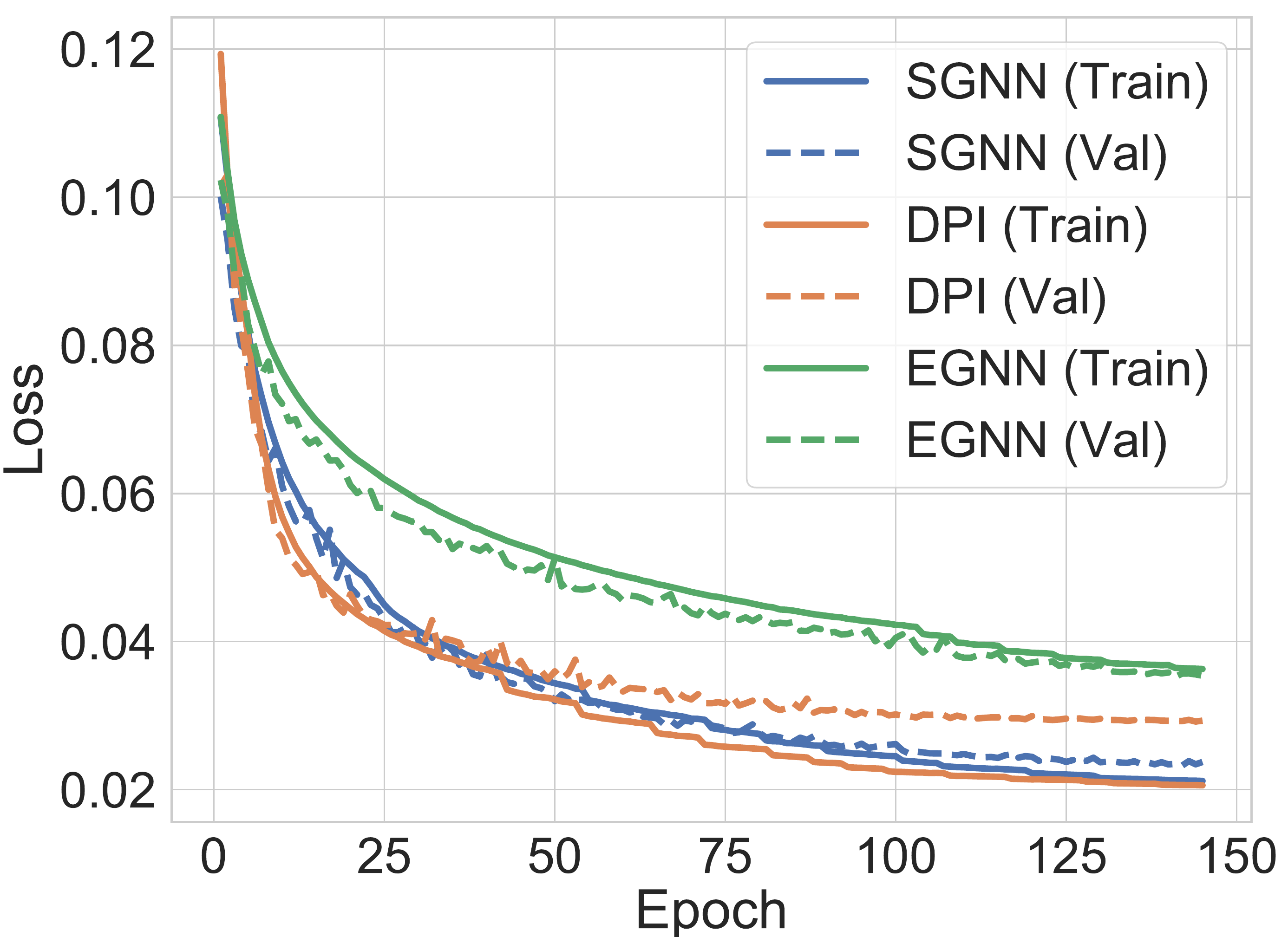}
    \caption{Learning curve comparisons on RigidFall. Left: $|\text{Train}|=200$; Right: $|\text{Train}|=500$.}
    \label{fig:learning_curve}
\end{figure}

We provide the learning curves including training loss and validation loss on RigidFall in Fig.~\ref{fig:learning_curve} with training data size 200 and 500, respectively. Our core observations here include \textbf{1.} The non-equivariant model DPI tends to suffer from \emph{overfitting}, since it lacks the inductive bias of symmetry. This can be observed from the curve that its training loss can reach a very low level (\emph{e.g.}, better than SGNN when $|\text{Train}|=200$), but the validation is not promising. There is generally a big gap between training and validation. Without the equivariance constraint, it may overfit the directions only existing in the training data and fail to generalize to validation set;
\textbf{2.} The E($3$)-equivariant model EGNN may underfit the training data, since it is restricted by an over-strong constraint that indeed fails to capture the real dynamics in the training data, though its generalization is desirable with a very small gap between training and validation; \textbf{3.} Our SGNN, by leveraging subequivariance, fits the training data well while yields strong generalization, achieving the lowest validation loss. These observations from the learning curves well align with our analyses and verify the efficacy of our design.





\subsection{Experiment on Hamiltonian-based NNs}
We augment SGNN and EGNN by a Hamiltonian integrator. Details of the implementation include:

\begin{itemize}
    \item We leverage a sum-pooling over the output scalar feature ($\mathbf{h}_i$) as the Hamiltonian of the system, i.e., $\mathcal{H} \in \mathbb{R}=\sum_{i=1}^N \mathbf{h}_i$.
    \item We employ a RK1 integrator to conduct Hamiltonian update, i.e., $(\dot{\vec{\mathbf{q}}}, \dot{\vec{\mathbf{p}}})=(\frac{\partial\mathcal{H}}{\partial{\vec{\mathbf{p}}}}, -\frac{\partial\mathcal{H}}{\partial{\vec{\mathbf{q}}}})$.
\end{itemize}

One thing worth noticing here is that we are assuming the particles possess uniform mass, so that $\vec{\mathbf{q}}, \vec{\mathbf{p}}$ can be derived from $\vec{\mathbf{x}}, \vec{\mathbf{v}}$, respectively. We name the variants as SGNN-H and EGNN-H ("H" stands for Hamiltonian), and evaluate them on Physion. The results are displayed in Table~\ref{tab:hamil}.

\begin{table}[htbp]
  \centering
  \caption{Comparison with the Hamiltonian-based variants.}
    \begin{tabular}{lcccccccc}
    \toprule
          & \multicolumn{1}{c}{Domino} & \multicolumn{1}{c}{Contain} & \multicolumn{1}{c}{Link} & \multicolumn{1}{c}{Drape} & \multicolumn{1}{c}{Support} & \multicolumn{1}{c}{Drop} & \multicolumn{1}{c}{Collide} & \multicolumn{1}{c}{Roll} \\
    \midrule
    EGNN~\cite{satorras2021en}  & 61.3  & 66.0    & 52.7  & 54.7  & 60.0    & 63.3  & 76.7  & 79.8 \\
    EGNN-H~\cite{satorras2021en,sanchez2019hamiltonian} & 52.0    & 58.1  & 54.0    & 54.3  & 51.1  & 54.7  & 75.7  & 75.3 \\
    SGNN  & 89.1  & 78.1  & 73.3  & 60.6  & 71.2  & 74.3  & 85.3  & 84.2 \\
    SGNN-H & 69.9  & 66.0    & 61.1  & 60.3  & 55.3  & 62.0    & 79.3  & 78.7 \\
    \bottomrule
    \end{tabular}%
  \label{tab:hamil}%
\end{table}%

\begin{table}[htbp]
  \centering
  \caption{Average training time per step (in seconds) on Physion Dominoes.}
    \begin{tabular}{cccc}
    \toprule
      EGNN~\cite{satorras2021en}    &  EGNN-H~\cite{satorras2021en,sanchez2019hamiltonian}     &    SGNN   &  SGNN-H \\
    \midrule
      0.08$\pm$0.01    &    0.44$\pm$0.02   &   0.11$\pm$0.02    & 0.48$\pm$0.03 \\
    \bottomrule
    \end{tabular}%
  \label{tab:hamil_time}%
\end{table}%

Adding Hamiltonian into EGNN and SGNN generally leads to detrimental performance. We speculate that it is probably due to the dissipative forces as well as highly complex interactions in Physion. Moreoever, as illustrated in Table~\ref{tab:hamil_time}, the Hamiltonian module brings significant computation overhead during training. This result suggests that it may not be beneficial to involve such strong physical inductive bias for the scenarios in Physion.

\subsection{Experiment on Steerable SE(2) GNN}
\label{sec:se2gnn}
We also implement a baseline that leverages the idea in~\cite{weiler2019general,cesa2021program} but extends from CNN to GNN. Indeed,~\cite{weiler2019general,cesa2021program}  are steerable CNNs, and these works have not offered available implementations on GNNs. We have tried our best to compare this idea with our model. Specifically, we implement ``Steerable-SE(2)-GNN'', that iterates the message passing as specified below. Consider the message computation for the edge $e_{ij}\in\mathcal{E}$ connecting node $i$ and $j$.
\begin{itemize}
    \item Compute the translation-invariant radial vector: $\vec{\mathbf{x}}_{ij}=\vec{\mathbf{x}}_i - \vec{\mathbf{x}}_j$.
    \item Project $\vec{\mathbf{x}}_{ij}$ onto $\vec{\mathbf{g}}$: $v\in\mathbb{R}=\frac{\vec{\mathbf{x}}_{ij}\cdot\vec{\mathbf{g}}}{\|\vec{\mathbf{g}} \|}$, and $\vec{\mathbf{u}}\in\mathbb{R}^2=((\vec{\mathbf{x}}_{ij}-v\vec{\mathbf{g}})\cdot \vec{\mathbf{m}}, ((\vec{\mathbf{x}}_{ij}-v\vec{\mathbf{g}})\cdot \vec{\mathbf{n}})$, where $\vec{\mathbf{m}}, \vec{\mathbf{n}}$ are two orthonormal bases vertical to $\vec{\mathbf{g}}$.
    \item Derive the type-0 message as $\mathbf{m}_{ij} = \text{MLP}_1(\sum_l w^{01}_l k^{01}_l(\vec{\mathbf{u}}) \cdot \vec{\mathbf{u}}, v, \|\vec{\mathbf{u}}\|, \mathbf{h}_i, \mathbf{h}_j)$.
    \item Derive the type-1 message as $\vec{\mathbf{M}}_{ij}=(\sum_l w_l^{10} k^{10}_l(\vec{\mathbf{u}})\mathbf{m}_{ij}+\sum_l w_l^{11} k_l^{11}(\vec{\mathbf{u}})\cdot\vec{\mathbf{u}})\cdot\text{MLP}_2(\mathbf{m}_{ij})$.
    \item Aggregate and update type-0 feature: $\mathbf{h}'_i=\text{MLP}_3(\sum_{j\in \mathcal{N}(i)}\mathbf{m}_{ij}, \mathbf{h}_i)$.
    \item Aggregate and update type-1 feature: $\vec{\mathbf{M}_i}=\sum_{j\in\mathcal{N}(i)}\vec{\mathbf{M}}_{ij}$,  $\vec{\mathbf{x}}_{i}'=\vec{\mathbf{x}}_{i} + \text{MLP}_4(\|\vec{\mathbf{M}}_{i}\|)\frac{\vec{\mathbf{M}}_{i}}{\|\vec{\mathbf{M}}_{i}\|+\epsilon}$.
\end{itemize}

Particularly, $w_l^{10}, w_l^{01}, w_l^{11}\in\mathbb{R}$ are the learnable coefficients and $k^{10}_l, k^{01}_l, k^{11}_l$ are the steerable kernel bases that transform irreps from type 1 to 0, type 0 to 1, and type 1 to 1, respectively (c.f. Table 8 in~\cite{weiler2019general} for more details); $\text{MLP}_2(\mathbf{m}_{ij})\in\mathbb{R}, \text{MLP}_4(\|\vec{\mathbf{M}}_{i}\|)\in\mathbb{R}$. It is proved that the above implementation is equivariant with respect to the subgroup $SO_{\vec{\mathbf{g}}}(3)$.

We compare Steer-SE(2)-GNN with EGNN, EGNN-S (the subequivariant version of EGNN), and SGNN on Physion in Table~\ref{tab:steer}.
\begin{table}[htbp]
  \centering
  \small
  \caption{Comparison with Steerable SE(2) GNN.}
    \begin{tabular}{lcccccccc}
    \toprule
          & \multicolumn{1}{c}{Dominoes} & \multicolumn{1}{c}{Contain} & \multicolumn{1}{c}{Link} & \multicolumn{1}{c}{Drape} & \multicolumn{1}{c}{Support} & \multicolumn{1}{c}{Drop} & \multicolumn{1}{c}{Collide} & \multicolumn{1}{c}{Roll} \\
    \midrule
    EGNN~\cite{satorras2021en}  & 61.3  & 66.0    & 52.7  & 54.7  & 60.0    & 63.3  & 76.7  & 79.8 \\
    Steer-SE(2)-GNN~\cite{weiler2019general} & 59.1  & 66.7  & 54.0    & 51.1  & 62.5  & 66.7  & 77.0    & 77.3 \\
    EGNN-S & 72.0    & 64.6  & 55.3  & 55.3  & 60.5  & 69.3  & 79.3  & 81.6 \\
    SGNN  & 89.1  & 78.1  & 73.3  & 60.6  & 71.2  & 74.3  & 85.3  & 84.2 \\
    \bottomrule
    \end{tabular}%
  \label{tab:steer}%
\end{table}%

Steer-SE(2)-GNN outperforms EGNN on 5 out of 8 tasks and obtains comparable results on the other 3 tasks, which indicates the reliability of our implementation. The reason why Steer-SE(2)-GNN is generally better than EGNN lies in the involvement of the gravity constraint. If considering this constraint as well, EGNN-S consistently surpasses Steer-SE(2)-GNN. Overall, our SGNN achieves the significantly best performance.
\section{More Visualizations}
Please refer to our Supplementary Video, presented at our project page \url{https://hanjq17.github.io/SGNN/}. 

Moreover, to further evaluate the generalization of different models toward unseen scenes and assess whether they properly learn the effect of gravity, we conduct extra experiments by applying a rotation around a non-gravity axis, resulting in such scenarios where dominoes are placed on an incline while gravity still points downwards vertically.

The video is also presented at our project page. It is worth noticing that all models are only trained with the original data (horizontal table with vertical gravity) and none of them have seen any scenario placed like these. We have the following observations.
\begin{itemize}
    \item Interestingly, SGNN well generalizes to these novel scenarios and reasonably simulates the effect of gravity. Particularly, the domino at the bottom starts to slide down along the table driven by gravity. The dominoes at the top reach an equilibrium between friction and gravity and keep still. The small bottle placed on the table also falls down due to gravity.
    \item EGNN, as an E(3)-equivariant model, does not perceive the changes in scenarios, producing the same trajectory as if the table is horizontal. GNS and DPI, by not incorporating rotation symmetry, do not properly learn the effect of gravity as well.
\end{itemize}

This experiment interestingly reveals that our SGNN is able to learn how gravity acts on physical dynamics effectively from data and can thus generalize to novel scenes, verifying the validity of our motivation and design of subequivariance.

\section{More Insights on Subequivariance}

In the paper we term subequivariance as a \emph{relaxation} of equivariance. In order to help understanding the position of our work, we provide more explanations and comparisons between full-equivariant models, non-equivariant models, and our subequivariant models.

\begin{figure}[htbp]
    \centering
    \includegraphics[width=0.98\textwidth]{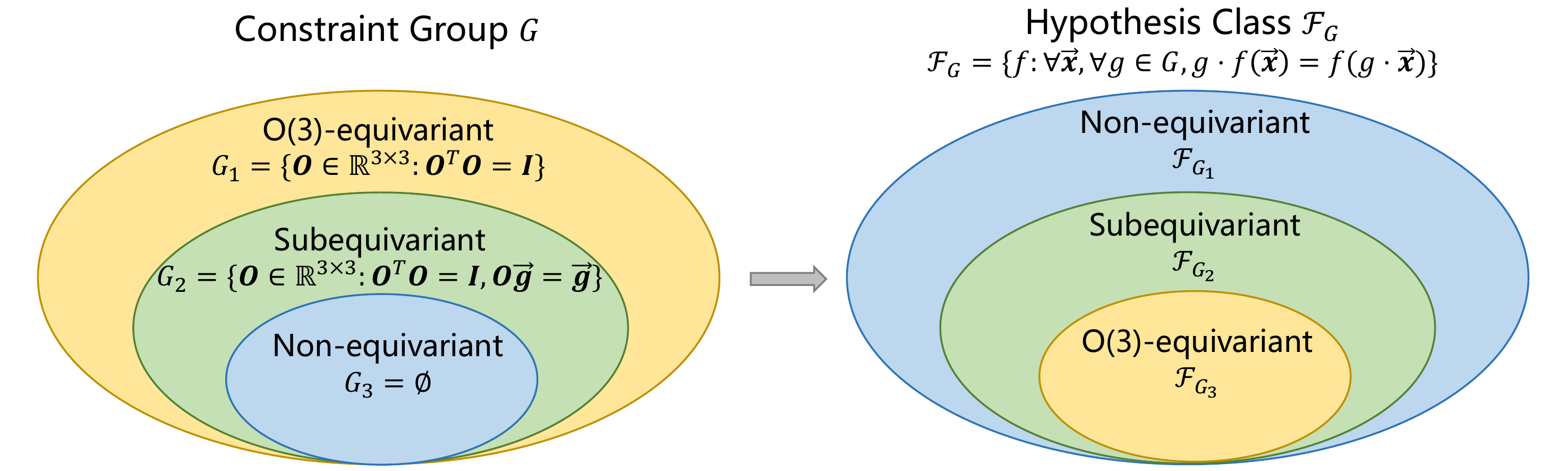}
    \caption{The comparison between O($3$)-equivariant, subequivariant, and non-equivariant models.}
    \label{fig:compare}
\end{figure}

As depicted in Fig.~\ref{fig:compare}, equivariance is exerted by a group $G$ that serves as a constraint over the corresponding possible functions, or so called the hypothesis class $\gF_G$. The functions in $\gF_G$ must satisfy that for any group element $g\in G$, the output of the function should be transformed the same way as the input by $g$, or formally, $\gF_G=\{f:\forall \vec{\vx},g\in G, g\cdot f(\vec{\vx})=f(g\cdot \vec{\vx})  \}$. Particularly, if two groups satisfy $G_1 \subseteq G_2$, it is straightforward to see that the corresponding hypothesis class $\gF_{G_2}\subseteq \gF_{G_2}$.

Specifically, we consider O($3$)-equivariant models with the constraint group $G_1=\{\mO\in\sR^{3\times 3}:\mO^\top\mO=\mI  \}$ including all orthogonal matrices, and non-equivariant models with $G_3=\emptyset$. Clearly the non-equivariant models possess a larger hypothesis class, which are usually easier to optimize during training. However, the drawback is the weaker generalization since the optimized function might not obey the proper constraint that is implied in the data. This is experimentally verifies by the training and validation curve of DPI in Fig.~\ref{fig:learning_curve}. The O($3$)-equivariant models, on the other hand, always satisfy the constraint $G$ and thus have a much smaller $\gF$. In the existence of gravity, the symmetry is violated in the vertical direction, and not all $g\in$ O($3$) should still serve as a constraint, but only those among $G_2 = \{\mO\in\sR^{3\times 3}:\mO^\top\mO=\mI, \mO\vec{\vg}=\vec{\vg} \}$. Therefore, $\gF_{G_1}$ becomes over-constrained, which significantly impedes the training (see the training curve of EGNN in Fig.~\ref{fig:learning_curve}). Our subequivariant model, instead, leverages $G_2$ as the constraint with $G_3\subseteq G_2\subseteq G_1$ and therefore $\gF_{G_1}\subseteq \gF_{G_2}\subseteq \gF_{G_3}$. With this proper relaxation, we expect the subequivariant models, equipped with the appropriate constraint, to have an ideal trade-off between training and generalization, which is also verified by the experimental results of SGNN and the learning curve in Fig.~\ref{fig:learning_curve}.

\end{document}